\documentclass[final,5p,times,twocolumn,authoryear]{elsarticle}

\usepackage{amsmath}
\usepackage{amssymb}
\usepackage{amsthm}
\usepackage{array}
\usepackage{mathrsfs}
\usepackage{bm}
\usepackage{dsfont}

\usepackage{graphicx}
\graphicspath{{Figures/}}
\usepackage{booktabs}
\usepackage{multirow}
\usepackage{makecell}
\usepackage{xcolor}
\usepackage{algorithm}
\usepackage{algpseudocode}
\usepackage{subcaption}

\usepackage[authoryear]{natbib}
\usepackage[unicode,hypertexnames=false]{hyperref}

\hypersetup{
  colorlinks=true,
  linkcolor=blue,
  citecolor=black,
  urlcolor=black,
  pdftitle={Hierarchical Empirical-Bayes Naive Bayes: Minimax Smoothing and
    Calibration with AODE Extension},
  pdfauthor={Nguyen Thai Anh, Truong Viet Vu, Tran Thien Thanh,
    Vo Nguyen Quoc Bao, and Ngo Hoang Tu}
}

\newtheorem{theorem}{Theorem}
\newtheorem{lemma}{Lemma}
\newtheorem{proposition}{Proposition}
\newtheorem{corollary}{Corollary}
\newtheorem{assumption}{Assumption}
\newtheorem{remark}{Remark}

\newcommand{\E}{\mathbb{E}}
\newcommand{\Prob}{\mathbb{P}}

\newcommand{\TV}{\mathrm{TV}}
\newcommand{\ECE}{\mathrm{ECE}}
\newcommand{\HEB}{\mathrm{HEB}}

\newcommand{\norm}[1]{\left\lVert #1 \right\rVert}

\newcommand{\Lap}{\mathrm{Lap}}

\allowdisplaybreaks
\newif\ifshowfigures
\showfigurestrue

\let\cite\citep

\usepackage{balance}

\journal{Knowledge-Based Systems}

\begin{document}

\begin{frontmatter}

\title{Hierarchical Empirical-Bayes Naive Bayes: Minimax Smoothing and Calibration with AODE Extension}

\author[inst1]{Nguyen Thai Anh}
\ead{anh.nt@vlu.edu.vn}
\author[inst1]{Truong Viet Vu}
\ead{vu.2274802011045@vanlanguni.vn}
\author[inst2]{Tran Thien Thanh}
\ead{thanh.tran@ut.edu.vn}
\author[inst1]{Vo Nguyen Quoc Bao}
\ead{bao.vnq@vlu.edu.vn}
\author[inst1]{Ngo Hoang Tu\corref{cor1}}
\ead{tu.nh@vlu.edu.vn}
\cortext[cor1]{Corresponding author.}

\affiliation[inst1]{
  organization={Faculty of Information Technology, Van Lang School of
  Technology, Van Lang University},
  city={Ho Chi Minh City},
  postcode={70000},
  country={Vietnam}
}

\affiliation[inst2]{
  organization={Institute of Information Technology and
  Electrical-Electronics Engineering, Ho Chi Minh City University of
  Transport},
  city={Ho Chi Minh City},
  postcode={700000},
  country={Vietnam}
}

\begin{abstract}
The Naive Bayes (NB) classifier remains a standard choice for categorical data, yet its widely used smoothing rules, such as Laplace, Lidstone, Krichevsky--Trofimov, and the $m$-estimate, all prescribe a fixed smoothing strength that ignores feature cardinality, sample size, and class imbalance, inducing a non-vanishing bias on modern high-cardinality tabular data.
%
% The Naive Bayes (NB) classifier remains a standard choice for categorical data, yet its widely used smoothing rules, such as Laplace, Lidstone, Krichevsky--Trofimov, and the $m$-estimate, all prescribe a fixed smoothing strength that ignores feature cardinality, sample size, and class imbalance. On modern tabular datasets with high-cardinality categorical features, this practice induces a non-vanishing bias that can dominate estimation error.
We propose hierarchical empirical-Bayes Naive Bayes (HEB-NB), in which each class-feature conditional probability is smoothed by a Dirichlet prior whose concentration is learned data-adaptively via Type-II maximum likelihood, enabling principled information sharing across classes while retaining closed-form inference. We further introduce HEB average one-dependence estimators (HEB-AODE), showing that the adaptive smoothing transfers cleanly to structural relaxations of NB.
% 
% We propose hierarchical empirical-Bayes Naive Bayes (HEB-NB), a categorical NB classifier in which each class-feature conditional probability is smoothed by a Dirichlet prior whose concentration is learned data-adaptively via Type-II maximum likelihood. This enables principled information sharing across classes while retaining closed-form inference and the same time and memory complexity as plain NB up to a constant factor. We further introduce HEB average one-dependence estimators (HEB-AODE), showing that the adaptive smoothing transfers cleanly to structural relaxations of NB.
% We first propose hierarchical empirical-Bayes Naive Bayes (HEB-NB), a categorical NB classifier in which each class-feature conditional probability mass function is smoothed by a Dirichlet prior whose concentration parameter is learned data-adaptively via Type-II maximum likelihood, enabling principled information sharing across classes while retaining closed-form inference.
% Theoretically, we establish a non-asymptotic $\ell_1$ error bound for HEB-NB matching the empirical-distribution minimax rate of discrete probability estimation plus a vanishing data-adaptive bias, in contrast to Laplace's non-vanishing bias in the high-cardinality regime. 
% A matching lower bound shows Laplace's bias is Laplace-tight at the simplex vertex, yielding a rigorous, finite-sample, risk-level strict separation between HEB-NB and Laplace that cannot be attributed to bound looseness. 
Theoretically, we establish a non-asymptotic $\ell_1$ error bound for HEB-NB matching the empirical-distribution minimax rate plus a vanishing data-adaptive bias, together with a matching Laplace-tight lower bound that yields a finite-sample, risk-level strict separation from Laplace. We further derive a plug-in excess Bayes-risk bound via total-variation tensorization and a population top-1 expected calibration error (ECE) corollary.
Empirically, across 31 UCI and OpenML benchmarks, HEB-NB attains the best average Friedman rank on probabilistic metrics, with up to 22.1\% log-loss reductions on high-cardinality datasets and consistent improvements of HEB-AODE over vanilla AODE. Combining HEB-NB with mutual-information weighting reduces top-1 ECE by 41\%--70\%, demonstrating substantial gains in probabilistic accuracy and calibration.
\end{abstract}

%% Highlights are condensed from existing manuscript claims; no new
%% scientific claim has been introduced.
% \begin{highlights}
% \item A closed-form, hyperparameter-free smoothed categorical NB classifier.
% \item A hierarchical Type-II ML smoother for AODE's pairwise conditionals.
% \item A rigorous, finite-sample, risk-level strict separation.
% \item HEB-M attains the best mean Friedman rank on log-loss and Brier.
% \item Combining HEB-NB with mutual-information weighting reduces top-1 ECE.
% \end{highlights}


\end{frontmatter}

%% Uncomment only when the journal requests a line-numbered review copy.
% \linenumbers

%% =====================================================================
\section{Introduction}\label{sec:intro}

The Naive Bayes (NB) classifier approximates the joint distribution of a label $Y$ and a feature vector {$X=(X_1,\dots,X_F)$} through the conditional-independence factorization {$\Prob(X\!\mid\! Y)=\prod_{f=1}^{F} \Prob(X_f\!\mid\! Y)$}.
Despite the well-known inadequacy of this assumption~\cite{hand2001idiot,zhang2004optimality}, NB remains a workhorse for categorical and binned tabular data due to its robustness in small-sample regimes and its transparent decision rule.
The recent NB literature has refined the classifier along two axes: 
(\textit{i}) attribute and instance weighting, which rescales each {$\log\hat \Prob(X_f\!\mid\! Y)$} in the posterior~\citep{he2025attribute,jiang2019class,zhang2021attribute, zhang2022fine,zhou2022adaptive}; and (\textit{ii}) structural relaxation, which adds limited dependencies among features, as in tree augmented NB~\cite{friedman1997bayesian}, average one-dependence estimators (AODE)~\cite{webb2005not}, and hidden NB~\cite{jiang2009hnb}.

% \IEEEPARstart{T}{he} Naive Bayes (NB) classifier approximates the joint
% distribution of a label $Y$ and a feature vector {$X=(X_1,\dots,X_F)$}
% through the conditional-independence factorisation
% {$\Prob(X\mid Y)=\prod_{f=1}^{F} \Prob(X_f\mid Y)$}.  Despite the well-known
% inadequacy of this assumption~\cite{hand2001idiot,zhang2004optimality},
% NB remains a workhorse for categorical and binned tabular data, its robustness in small-sample
% regimes, and its transparent decision rule.  
% The recent NB literature
% has refined the classifier along two axes: 
% {attribute and instance
% weighting}, which rescales each {$\log\hat \Prob(X_f\mid Y)$} in the
% posterior~\cite{jiang2019class,zhang2021attribute,zhang2022fine,
% he2025attribute,zhou2022adaptive}; and {structural relaxation},
% which adds limited dependencies among features, as in
% tree augmented NB~\cite{friedman1997bayesian}, average one-dependence estimators (AODE)~\cite{webb2005not}, and hidden NB~\cite{jiang2009hnb}.
 
% \emph{The Neglected Third Axis}:
Both refinement axes treat the conditional probability estimator {$\hat \Prob(X_f\!\mid\! Y)$} as a fixed input.  
The estimator itself has received remarkably little attention since the 1990s.  
Every realistic NB implementation applies smoothing such that empirical counts are augmented with non-negative pseudo-counts before normalization, both to avoid zero-probability pathologies and to stabilize rare-category estimates.  
The dominant smoothers, such as Laplace's add-one rule, Lidstone's
$\alpha$-smoothing~\cite{lidstone1920note},
Krichevsky--Trofimov (KT) \cite{krichevsky1981performance}, and Cestnik's
$m$-estimate~\cite{cestnik1990estimating}, all prescribe a smoothing strength fixed a priori, ignoring feature cardinality {$K_f$}, within-class sample size $N_c$, and the agreement between the implicit prior mean and the true class-conditional probability mass function (PMF).  
The single notable departure is the high-probability-binning (HPB) procedure of Jambeiro Filho and Wainer~\cite{jambeiro2008hpb}, which anticipates the empirical-Bayes idea for general Bayesian-network nodes but predates the modern minimax PMF-estimation literature and offers no risk or calibration analysis.
 
This neglect becomes consequential on modern tabular data with
high-cardinality features, e.g., uniform-resource-locator hashes, zone-improvement-plan codes, product identifications, internet-protocol blocks, sparse one-hot text tokens, where {$K_f$} routinely reaches the thousands or tens of thousands per
feature~\cite{micci2001preprocessing}.  
In such regimes $K_f\gtrsim N_c$ and pseudo-counts dominate empirical counts, so that Laplace smoothing incurs a deterministic, non-vanishing bias whose order matches the sampling-variance scale and saturates at $\Theta(1)$ once $K_f$ exceeds $N_c$. Read against the empirical-distribution upper bound of~\citet[Theorem~1]{han2015minimax}, this is a strict, non-asymptotic efficiency loss that Section~\ref{sec:theory} quantifies precisely.

% $\ell_1$ bias of order {$K_f/(N_c+K_f)$}, comparable to the
% sampling-variance scale {$\sqrt{K_f/N_c}$} once {$K_f\gtrsim N_c/4$} and
% saturating at $\Theta(1)$ thereafter.  Read against the
% empirical-distribution upper bound
% {$\sqrt{(K_f-1)/N_c}$}~\cite[Theorem~1]{han2015minimax}, this is a strict,
% non-asymptotic efficiency loss.

% We close this gap with a hierarchical empirical-Bayes (HEB) construction.  
% For each class $c$ and feature {$f$}, we model the class-conditional PMF as $\boldsymbol{\theta}_{c,f}\sim\mathrm{Dirichlet}(m_{c,f}\bar{\boldsymbol{\theta}}_f)$.  
% The Dirichlet is the standard prior over probability vectors and is
% conjugate to the multinomial likelihood, so the posterior remains
% available in closed form.
% The concentration $m_{c,f}>0$ acts as a smoothing strength such that large $m_{c,f}$ shrinks the posterior heavily
% toward $\bar{\boldsymbol{\theta}}_f$, and $m_{c,f}\!\to\!0$ recovers the unsmoothed maximum-likelihood (ML) estimator.  
% The concentration is learned from data by Type-II ML, evidence maximization with $\boldsymbol{\theta}_{c,f}$ integrated out~\cite{minka2000estimating}, which selects the smoothing strength that best explains the observed counts under the chosen prior, with no user-supplied hyperparameter. 

We close this gap with a hierarchical empirical-Bayes (HEB) construction. For each class $c$ and feature $f$, we model the class-conditional PMF as $\boldsymbol{\theta}_{c,f}\sim\mathrm{Dirichlet}(m_{c,f}\bar{\boldsymbol{\theta}}_f)$, exploiting Dirichlet-multinomial conjugacy for closed-form posteriors. The concentration $m_{c,f}>0$ controls shrinkage toward $\bar{\boldsymbol{\theta}}_f$ ($m_{c,f}\to 0$ recovers the unsmoothed maximum-likelihood (ML) estimator) and is learned from data by Type-II ML~\cite{minka2000estimating}, with no user-supplied hyperparameter.
The prior mean $\bar{\boldsymbol{\theta}}_f$ is either uniform (HEB-U) or the pooled
marginal $\hat \Prob(X_f)$ (HEB-M); the latter borrows information across classes in the spirit of target encoding (TE) \cite{micci2001preprocessing}, but inside the classifier rather than as preprocessing.
% \noindent\textbf{Contributions.}
The main contributions of this work are summarized as follows:

% \begin{enumerate}\vspace{-0.15cm}
% \item 
\textit{1) HEB-NB Estimator (Section~\ref{sec:HEB-NB})}:
A closed-form, hyperparameter-free smoothed categorical NB classifier whose per-(class, feature) Dirichlet concentration is learned by Type-II ML (Algorithm~\ref{alg:heb-nb}), with two natural prior-mean variants HEB-U and HEB-M. 
Limit cases recover Laplace smoothing and the full-shrinkage endpoint of Micci-Barreca TE, so HEB-NB unifies both within a single hierarchical-Bayes family while preserving the time and memory complexity of categorical NB up to a constant factor.
 
% \item 
\textit{2) HEB-AODE Structural Extension (Section~\ref{sec:hebaode})}:
A hierarchical Type-II ML smoother for AODE's pairwise conditionals, in which the concentration is shared across the
super-parent's instantiations and learned by a multi-document extension of Minka's single-document fixed-point iteration.  
HEB-AODE recovers vanilla AODE as a limit case and demonstrates that the HEB smoothing improvement transfers cleanly across the structural axis of NB refinement.

% \item 
\textit{3) Non-Asymptotic Theoretical Analysis
(Section~\ref{sec:theory})}:
We establish
(\textit{i}) a pathwise $\ell_1$ upper bound for HEB-NB (Theorem~\ref{thm:smoothing}) that matches the empirical-distribution rate of~\citet[Theorem~1]{han2015minimax} plus a data-adaptive bias, in contrast to Laplace's deterministic non-vanishing bias;
(\textit{ii}) a rigorous sanity bound on the HEB shrinkage factor (Lemma~\ref{lem:typeiiml-rate}) via the deterministic clamp from Assumption~\ref{Assumption:3}, ensuring that the bias vanishes
asymptotically as the within-class sample size grows;
(\textit{iii}) a matching lower bound for Laplace (Theorem~\ref{thm:laplace-lower}) proving its bias is Laplace-tight at the simplex vertex; 
(\textit{iv}) a rigorous, finite-sample, risk-level strict separation between HEB-NB and Laplace (Corollary~\ref{cor:strict-separation}) that cannot be attributed to bound looseness in either estimator; 
and (\textit{v}) an excess Bayes-risk bound via joint total-variation (TV) tensorization (Corollary~\ref{cor:bayesrisk}), which does not require any uniform lower bound on $\Prob(X)$, a property essential in the high-cardinality regime, together with a corresponding population top-1 expected calibration error (ECE) bound (Theorem~\ref{thm:calibration}).

\textit{4) Comprehensive Empirical Study
(Section~\ref{sec:experiments})}:
Across 31 UCI and OpenML datasets (including three high-cardinality benchmarks with up to $16{,}137$ categories per feature), HEB-M attains the best mean Friedman rank on log-loss and Brier, and significantly dominates 13 of 24 baseline--metric pairs under the Nemenyi protocol at $\alpha = 0.05$. 
HEB-AODE significantly improves over vanilla AODE on F1, log-loss, and Brier (paired Wilcoxon $p \le 0.029$ on each), with log-loss reductions of up to 22.1\% on click-prediction. Combining HEB-M with sqrt-mutual-information (sqrt-MI) weighting cuts top-1 ECE by 41\%--70\% versus Laplace on the high-cardinality benchmarks, demonstrating that the smoothing improvement is orthogonal to both the weighting and structural axes of NB refinement.

% Across $31$ UCI and OpenML datasets
% ($K_{\max}\!\in\![5\,760, 16\,137]$ on the three high-cardinality
% benchmarks), HEB-M attains the best mean Friedman rank on log-loss and
% Brier ($\bar r=2.26$ each) and, under the standard Friedman--Nemenyi
% multi-method comparison protocol of
% Dem\v{s}ar~\cite{demvsar2006statistical}, significantly dominates
% $13$ of $24$ baseline-metric pairs at $\alpha=0.05$.  HEB-AODE
% significantly improves over vanilla AODE on F1, log-loss, and Brier
% (paired Wilcoxon $p\le 0.029$ on each), with log-loss reductions of
% $9.7\%$ on amazon-employee and $22.1\%$ on click-prediction.
% Combining HEB-M with sqrt-mutual-information weighting (HEB-M+GR)
% cuts top-1 ECE by $41\%$ to $70\%$ versus Laplace on the
% high-cardinality benchmarks, demonstrating empirically that the
% smoothing improvement is orthogonal to both the weighting and the
% structural axes of NB refinement.
% \end{enumerate}

The remainder of the paper is organized as follows.
Section~\ref{sec:related} reviews related work;
Section~\ref{sec:estimator} develops the proposed estimators;
Section~\ref{sec:theory} provides their theoretical analysis;
Section~\ref{sec:experiments} reports the empirical study and limitations;
and Section~\ref{sec:conclusion} concludes the paper.

% The remainder of the paper reviews related work
% (Section~\ref{sec:related}), develops the estimators
% (Section~\ref{sec:estimator}) and their theory
% (Section~\ref{sec:theory}), reports the empirical study
% (Section~\ref{sec:experiments}), and discusses limitations (Sections~\ref{sec:discussion}), and
% concludes the paper (Sections~\ref{sec:conclusion}).

\textit{Notational Definitions}:
$\mathbb{R}_{\ge 0}^{K}$ denotes the set of $K$-dimensional nonnegative real vectors.
$\Prob(\cdot)$ is the probability function.
$\norm{\cdot}_1$ and $\mathbb{E}(\cdot)$ denote the $\ell_1$-norm and expectation operators, respectively.
$a \wedge b = \min\{a,b\}$ denotes the truncation at level $b$.
All logarithms are to the natural base $e$.
We write $K \gtrsim N$ if $K \ge cN$ for some constant $c > 0$, and $K \asymp N$ if $c_1 N \le K \le c_2 N$ for some constants $c_1, c_2 > 0$, both for sufficiently large $N$.
Moreover, $o(\cdot)$, $\mathcal{O}(\cdot)$, $\tilde{\mathcal{O}}(\cdot)$, $\Omega(\cdot)$, and $\Theta(\cdot)$ denote asymptotic negligibility, asymptotic upper bounds, asymptotic upper bounds up to logarithmic factors, asymptotic lower bounds, and tight asymptotic order, respectively.

% \textit{Notational Definitions}:
% $\mathbb{R}_{\ge 0}^{K}$ denotes the set of nonnegative real vectors of dimension $K$.
% $\Prob(\cdot)$ denotes the probability function.
% $\inf_{\mathcal X}$ denotes the infimum (greatest lower bound) over all measurable values of a mapping function $\mathcal{X}$.
% All logarithms in this paper are to the natural base $e$ (i.e., $\log \equiv \ln$), as is standard for the Dirichlet-multinomial marginal likelihood and digamma-function analysis.
% $\norm{\cdot}_1$ and $\mathbb{E}(\cdot)$ denote the $\ell_1$-norm and expectation operators, respectively.
% %
% $a \wedge b = \min\{a,b\}$ denotes the truncation at level $b$.
% %
% $K \gtrsim N$ and $K \asymp N$ denote that $K \geq cN$ 
% for some constant $c > 0$ (i.e., $K$ is of the same order 
% as or larger than $N$), and $c_1 N \leq K \leq c_2 N$ 
% for some constants $c_1, c_2 > 0$ (i.e., $K$ is of the 
% same order as $N$), respectively, for sufficiently large $N$.
% %
% Furthermore, $o(\cdot)$, $\mathcal{O}(\cdot)$, $\tilde{\mathcal{O}}(\cdot)$,
% $\Omega(\cdot)$, and $\Theta(\cdot)$ denote asymptotic negligibility,
% asymptotic upper bounds, asymptotic upper bounds up to logarithmic factors,
% asymptotic lower bounds, and tight asymptotic order, respectively.

% all logarithms in this paper are to the base 10.
 % $K \gtrsim N$ denotes that $K$ is of the same order as $N$ or larger, i.e.,
% there exists a constant $c>0$ such that $K \ge c\,N$ for sufficiently large $N$.

%=====================================================================
% \vspace{-0.25cm}
\section{Related Work}\label{sec:related}

\textit{Smoothing for Categorical NB}: The earliest smoothers for NB are Laplace's add-one rule, Lidstone's $\alpha$-smoothing \cite{lidstone1920note}, and the universal KT mixture~\cite{krichevsky1981performance}. 
Cestnik's $m$-estimate~\cite{cestnik1990estimating} introduces a single concentration parameter and a fixed prior mean, but the concentration must be set manually. 
Good--Turing smoothing~\cite{good1953population} has been used for NB on text but, to our knowledge, without a finite-sample analysis. 
None of these methods adapt the smoothing strength per (class, feature) inside the NB classifier. 
The closest precedent is the HPB procedure of \citet{jambeiro2008hpb}, which introduces Type-II hyperparameter selection for Bayesian-network nodes with high-cardinality parents.
% HPB anticipates the empirical-Bayes idea in a BN-CPT context and includes plain NB as one of the structures it benchmarks against. 
Our work differs from HPB in two specific respects:
(\textit{i})~HPB targets the more general high-cardinality-parent setting and reports no excess-risk or minimax-rate analysis, whereas we instantiate Type-II ML directly in plain NB and provide non-asymptotic risk and calibration bounds (Theorems~\ref{thm:smoothing}--\ref{thm:calibration}); 
(\textit{ii})~the HPB analysis predates the modern minimax-PMF-estimation literature~\cite{han2015minimax}, from which our framing borrows. Empirical-Bayes TE~\cite{micci2001preprocessing} pools categories toward the marginal mean but is used as a feature-engineering preprocess before logistic regression or trees, with no theoretical analysis. Outside NB, hierarchical Dirichlet smoothing is common in topic models~\cite{blei2003latent} and language modelling~\cite{mackay1995hierarchical}.

\textit{Attribute and Instance Weighting}: 
A large recent literature reweights NB factors using mutual information, gain ratio, or class-conditional correlation: CFW~\cite{zhang2016two}, 
CAWNB~\cite{jiang2019class}, 
CAVWNB~\cite{zhang2020class}, 
AIWNB~\cite{zhang2021attribute},
FTAWNB~\cite{zhang2022fine}, 
ATFNB~\cite{zhou2022adaptive}, 
A$^2$WNB~\cite{zhang2022attribute}, 
FNB~\cite{hue2024fractional}, 
and AG-NBC~\cite{he2025attribute}. 
These methods are orthogonal to ours: they reweight {$\log \hat \Prob(X_f \!\mid\! Y)$} without changing how the true probability is estimated. 
% HEB-NB can be combined with any of them.

\textit{Structural Relaxation}: 
AODE~\cite{webb2005not}, tree augmented NB \cite{friedman1997bayesian}, hidden NB~\cite{jiang2009hnb}, and KDB~\cite{sahami1996learning} relax the Naive independence assumption by adding parents to attributes. 
These methods retain Laplace-style smoothing internally and would also benefit from HEB. 
Recent vine-copula NB classifiers~\cite{csahin2025vine} go further but at quadratic pair-copula cost and without consistency rates. 
Structural and smoothing improvements are orthogonal axes of NB refinement: HEB-NB advances the smoothing axis; AODE advances the structural axis. 
AODE has substantial memory overhead for high-cardinality features and is included in Section~\ref{sec:expt:main} as a reference baseline using a memory-capped variant. 
The combination of both axes is realized by our proposed HEB-AODE in Section~\ref{sec:hebaode}.

% relax the Naive independence assumption by adding parents to attributes. These methods retain Laplace-style smoothing internally and would also benefit from HEB. Recent vine-copula NB classifiers~\cite{csahin2025vine} go further but at {$\mathcal{O}(F^2)$} pair-copula cost and without consistency rates. 
% {Structural} and {smoothing} improvements are {orthogonal axes} of NB refinement: HEB-NB advances the smoothing axis (Sections~\ref{sec:estimator},~\ref{sec:theory}); AODE advances the structural axis. AODE requires {$O(F^2 \cdot K_{\max}^2 \cdot C)$} memory to store {$\Prob(X_j \mid Y, X_i)$} for every super-parent $i$ and child $j$, and is infeasible on high-cardinality features. We provide AODE as a reference baseline in Section~\ref{sec:expt:main} using a paired Wilcoxon test and a memory-capped variant; AODE wins on every metric, as expected from its richer model class. The combination of both axes is realised by \textbf{HEB-AODE} (Section~\ref{sec:hebaode}); the empirical comparison against vanilla AODE is reported in Section~\ref{sec:expt:main}.

\textit{Calibration of NB}:
NB is famously over-confident \cite{niculescu2005predicting}; the standard correction is post-hoc Platt scaling \cite{platt1999probabilistic}, isotonic regression \cite{zadrozny2002transforming}, or temperature scaling \cite{guo2017calibration}. 
Moreover, \citet{perez2024risk} propose a risk-based iterative recalibrator. Our work is complementary: HEB-NB improves the {upstream} probability estimator, whereas calibration corrects the {downstream} posterior.

% \color{blue}

% Abbreviations,, total-variation (TV), ECE

% \color{black}

%% =====================================================================
% \vspace{-0.25cm}
\section{Proposed HEB-NB and HEB-AODE Estimators}
\label{sec:estimator}

\autoref{fig:heb-overview} summarizes the training and prediction workflow of HEB-NB and its HEB-AODE extension. From the categorical training data, HEB-NB forms the class counts and the category-count vector for each class--feature pair, and constructs the feature-specific prior mean using either HEB-U (uniform) or HEB-M (pooled marginal). The concentration for each class--feature pair is learned by Type-II ML or set by the low-count fallback in Algorithm~\ref{alg:heb-nb}. The resulting empirical-Bayes posterior means are stored in log form with the estimated class priors. At test time, NB log scores are normalized into posterior probabilities for evaluation, and prediction selects the class with the highest posterior probability. HEB-AODE applies the same smoother to AODE's child-conditional factors, as detailed in Section~\ref{sec:hebaode}.

\begin{figure*}[!t]
\centering
{\includegraphics[width=\textwidth]{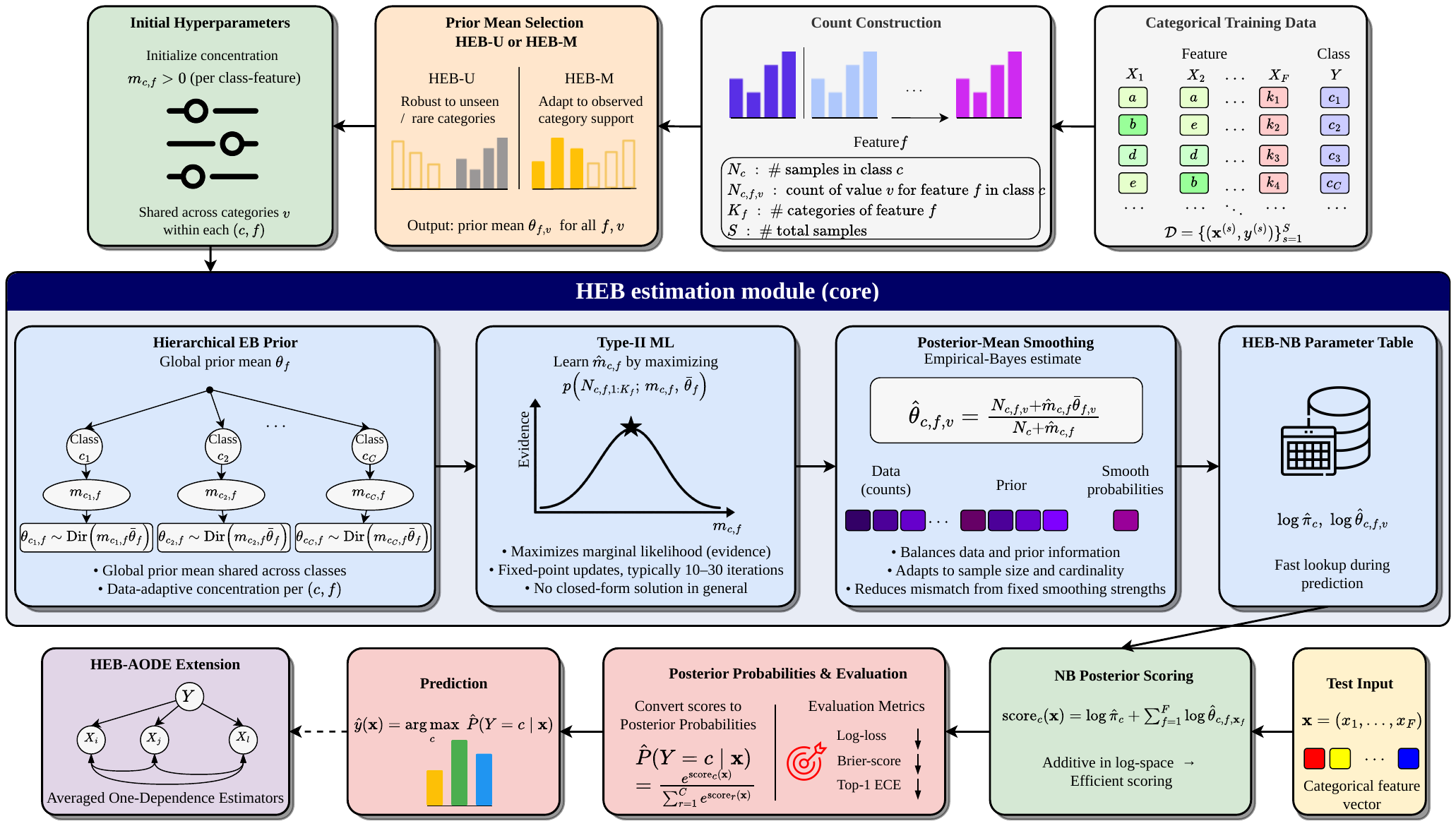}}
\caption{Training and prediction workflow of HEB-NB and its HEB-AODE extension.}
\label{fig:heb-overview}
\end{figure*}

\subsection{The Proposed HEB-NB Estimator}
\label{sec:HEB-NB}

\subsubsection{Categorical NB Model}

Let $\mathcal{D} = \{(\mathbf{x}^{(s)}, y^{(s)})\}_{s=1}^{S}$ denote a training dataset of $S$ samples, where $y^{(s)} \in \{1, \dots, C\}$ is a categorical class label from a set of $C$ classes, and $\mathbf{x}^{(s)} = (x_1^{(s)}, \dots, x_F^{(s)})$ is a vector of $F$ categorical features.
% \footnote{Unless otherwise stated, 
% $s \in \{ 1,\ldots,S \}$, 
% $c \in \{ 1,\ldots,C \}$,
% $f \in \{ 1,\ldots,F \}$,
% and
% $v \in \{ 1,\ldots,K_f \}$
% are assumed throughout this manuscript.}
%
Let $N_c := \left|\{s : y^{(s)} = c\}\right|$ denote the number of training samples belonging to class $c$, and let $N_{c,f,v} := \big|\{s : y^{(s)} = c,\; x_f^{(s)} = v\}\big|$ denote the joint count for class $c$, feature $f$, and category $v$.
Each feature $x_f^{(s)}$ takes values in a finite alphabet $\{1,\ldots,K_f\}$ of size $K_f$.
Unless otherwise stated, 
$s \in \{ 1,\ldots,S \}$, 
$c \in \{ 1,\ldots,C \}$,
$f \in \{ 1,\ldots,F \}$,
and
$v \in \{ 1,\ldots,K_f \}$
are assumed throughout this manuscript.
By construction, for every $(c,f)$ pair, $\sum_{v=1}^{K_f} N_{c,f,v} = N_c$.
%%%
To place the NB classifier in a decision-theoretic framework,
let $\mathcal{X} := \prod_{f=1}^{F} \{1,\dots,K_f\}$
denote the Cartesian product of the feature alphabets.
Let $X = (X_1,\dots,X_F)$ be a random feature vector taking values in
$\mathcal{X}$, and let $Y \in \{1,\dots,C\}$ be the corresponding random class label.
The categorical NB classifier predicts the class label of an input $\mathbf{x} = (x_1, \dots, x_F)$ under the conditional-independence assumption $\Prob(\mathbf{x} \!\mid\! Y) = \prod_{f=1}^F \Prob(X_f = x_f \!\mid\! Y)$, such that \cite[Eq.~(13.4)]{manning2008introduction}
% \hl{[Ref is needed]}
% with discrete labels $y^{(i)} \in \{1, \dots, C\}$ and feature vectors $x^{(i)} = (x_1^{(i)}, \dots, x_d^{(i)})$ where each $x_j$ takes values in a finite alphabet of size $K_j$. 
% Let $N_c := |\{i : y^{(i)} = c\}|$ and $N_{cjv} := |\{i : y^{(i)} = c, x_j^{(i)} = v\}|$ denote the marginal class count and joint (class, feature, value) count, respectively. 
% The categorical Naive Bayes posterior is
% \hl{[Ref is needed]}
\begin{equation}
\hat y(\mathbf{x}) = \arg\max_c \big[\! \log (\hat\pi_c) + \sum\nolimits_{f} \log (\hat\theta_{c,f,x_f}) \big],
\label{eq:nb-posterior}
\end{equation}
where $\hat\pi_c = (N_c + 1)/(S + C)$ is the Laplace-smoothed class prior,
and $\hat\theta_{c,f,v} := \hat \Prob(X_f = v \mid Y = c)$ denotes the estimated class-conditional PMF for each $v$, evaluated at $v = x_f$ in~\eqref{eq:nb-posterior}.
Notably, we use $\hat\pi_c$, $\hat\theta_{c,f,v}$, and $\hat \Prob(\cdot)$ to denote estimates of $\pi_c := \Prob(Y=c)$, $\theta_{c,f,v} := \Prob(X_f = v \!\mid\! Y = c)$, and $\Prob(\cdot)$, respectively, since their true values are unknown and are instead estimated from the training data.
% $\hat{\Prob}(\cdot)$ represents an estimate of the corresponding PMF learned from the training data $\mathcal D$.

For each pair $(c,f)$,
${{\boldsymbol{\theta }}}_{c,f} :=  [\theta_{c,f,1},\ldots,\theta_{c,f,K_f}]$ takes values in the probability simplex
$\Delta^{K_f-1} \triangleq \{\mathbf{p} = [p_1,\ldots,p_{K_f}] \in \mathbb{R}_{\ge 0}^{K_f} : \sum_{v=1}^{K_f} p_v = 1\}$.
The zero-one misclassification risk of a classifier $g : \mathcal{X} \to \{1, \dots, C\}$ is defined as
$R(g) := \Prob\bigl(g(X) \neq Y\bigr)$.
% , where the deterministic classifier $g$ is any function such that $g : \mathcal{X} \to \{1, \dots, C\}$.
% Accordingly, the Bayes risk and Bayes-optimal classifier are given, respectively, as
% \begin{align}
%     R^\star &:= \inf\nolimits_g R(g) = R(g^\star),\\
%     g^\star(x) &= \arg\max\nolimits_c \Prob(Y = c \mid X = x), \,\,\forall  x \in \mathcal X.
% \end{align}
Accordingly, the Bayes risk and Bayes-optimal classifier are given by
$R^\star := \inf\nolimits_g R(g) = R(g^\star)$ 
and
$g^\star(x) = \arg\max\nolimits_c \Prob(Y = c \mid X = x), \,\,\forall  x \in \mathcal X$, respectively.

% \textcolor{red}{Check until this...}

% The $0$--$1$ misclassification risk of a classifier
% $g : \mathcal{X} \to \{1, \dots, C\}$ is defined as
% $R(g) := \Prob(g(X) \neq Y)$.
% The Bayes risk is
% $R^\star := \inf_g R(g) = R(g^\star)$,
% where the Bayes-optimal classifier is
% $g^\star(x) = \arg\max_c \Prob(Y = c \mid X = x)$.

% , and the conditional pmf $\hat P(X_j \mid Y=c) \in \Delta^{K_j-1}$ (with $\Delta^{K-1} := \{p \in \mathbb{R}_{\ge 0}^{K} : \sum_v p_v = 1\}$ the probability simplex) must be estimated for every $(c,j)$ pair. The 0-1 misclassification risk of any classifier $g : \mathcal{X} \to \{1, \dots, C\}$ is $R(g) := \Prob(g(X) \ne Y)$, the Bayes risk is $R^\star := \inf_g R(g) = R(g^\star)$ with $g^\star(x) = \arg\max_c P(Y = c \mid X = x)$.

% \vspace{-0.25cm}
\subsubsection{The HEB-NB Estimator}\label{sec:hebnb-estimator}

To estimate $\theta_{c,f,v}$ in \eqref{eq:nb-posterior}, we adopt a Bayesian generative modeling approach.
Specifically, for each class-feature pair $(c,f)$, we model
${{\boldsymbol{\theta }}}_{c,f}$ as a random variable drawn from a Dirichlet prior, i.e.,
${{\boldsymbol{\theta }}}_{c,f} \sim \mathrm{Dirichlet}(m_{c,f} \,\bar{\boldsymbol{\theta}}_{f})$,
where $\bar{\boldsymbol{\theta}}_{f} := [\bar\theta_{f,1},\ldots,\bar\theta_{f,K_f}] \in \Delta^{K_f-1}$ is a \emph{prior mean} shared across classes, and $m_{c,f} > 0$ is a \emph{concentration parameter} controlling the strength of shrinkage.
Given the observed count vector
$\mathbf{n}_{c,f} = (N_{c,f,1}, \dots, N_{c,f,K_f})$, which is $\mathrm{Multinomial}(N_c, \boldsymbol{\theta}_{c,f})$ conditional on $\boldsymbol{\theta}_{c,f}$ by the categorical NB model, the Dirichlet-multinomial conjugacy yields the closed-form posterior
$\boldsymbol{\theta}_{c,f}\mid \mathbf{n}_{c,f} \sim \mathrm{Dirichlet}(m_{c,f}\bar{\boldsymbol{\theta}}_{f} + \mathbf{n}_{c,f})$.
% \begin{align}
%     \boldsymbol{\theta}_{c,f}\mid \mathbf{n}_{c,f} \sim \mathrm{Dirichlet}(m_{c,f}\bar{\boldsymbol{\theta}}_{f} + \mathbf{n}_{c,f}).
% \end{align}
The posterior mean of ${{\boldsymbol{\theta }}}_{c,f}$ therefore admits the element-wise closed form of
\begin{equation}
\hat\theta_{c,f,v}=\hat \Prob_{\HEB}(X_f = v \mid Y = c) = \frac{N_{c,f,v} + m_{c,f}\,\bar\theta_{f,v}}{N_c + m_{c,f}}.
\label{eq:heb}
\end{equation}

Two natural choices of the prior mean $\bar{\boldsymbol{\theta}}_{f}$ give rise to two variants of HEB-NB in \eqref{eq:heb}.
Specifically, a uniform prior, referred to as \textbf{HEB-U}, is given by
\begin{align} \label{eq:HEB-U}
    \bar\theta_{f,v} = 1/K_f, \,\, \forall v,
\end{align}
whereas a pooled marginal prior, referred to as \textbf{HEB-M}, is defined as
\begin{align}\label{eq:HEB-M}
    \bar\theta_{f,v} = \hat \Prob(X_f = v) 
    = \sum\nolimits_c N_{c,f,v}/S.
\end{align}

% \begin{remark}[Empirical-Bayes Data Dependence]\label{Remark:HEB-NB}
% % {\noindent\emph{Remark on empirical-Bayes data dependence.}
% The HEB-M prior mean $\bar{\boldsymbol{\theta}}_f$ in~\eqref{eq:HEB-M} is computed from the same training data $\mathcal{D}$ used to fit $(N_{c,f,v}, \hat m_{c,f})$ for each class $c$. 
% This is a defining characteristic of empirical Bayes: the prior is data-driven rather than specified a priori, and the resulting procedure is therefore not a strict Bayesian method.
% The information sharing induced by pooling across classes is what gives HEB-M its statistical advantage in the high-cardinality regime.
% For analytical tractability, however, the theoretical analysis in Section~\ref{sec:theory} treats $\bar{\boldsymbol{\theta}}_f$ as conditionally fixed when bounding the per-class HEB-NB error (see Remark~\ref{remark:NotationalNote}).
% \end{remark}

\begin{remark}[Empirical-Bayes Data Dependence]\label{Remark:HEB-NB}
The HEB-M prior mean $\bar{\boldsymbol{\theta}}_f$ in~\eqref{eq:HEB-M} is computed from the same training data $\mathcal{D}$ used to fit $(N_{c,f,v}, \hat{m}_{c,f})$, making the procedure empirical-Bayes rather than strictly Bayesian. The cross-class pooling drives HEB-M's advantage in the high-cardinality regime. For analytical tractability, Section~\ref{sec:theory} treats $\bar{\boldsymbol{\theta}}_f$ as conditionally fixed when bounding the per-class error (see Remark~\ref{remark:NotationalNote}).
\end{remark}

% \textcolor{red}{Check until this...}
% We model each class-conditional pmf $\theta_{cj} := P(X_j \mid Y=c) \in \Delta^{K_j-1}$ as drawn from a Dirichlet prior:
% where $\bar\theta_j \in \Delta^{K_j-1}$ is a \emph{prior mean} and $m_{cj} > 0$ is a \emph{concentration}. 
% The posterior mean given the count vector $N_{cj\cdot} = (N_{cj1}, \dots, N_{cjK_j})$ is
% Two natural choices for $\bar\theta_j$ define two variants of HEB-NB:
% \begin{itemize}
% \item \textbf{HEB-U} uses the uniform prior $\bar\theta_{j,v} = 1/K_j$;
% \item \textbf{HEB-M} uses the marginal pooled prior $\bar\theta_{j,v} = \hat P(X_j = v) = (\sum_c N_{cjv})/n$.
% \end{itemize}

%===============
% \subsection{Type-II Maximum-Likelihood Estimation of $m_{c,f}$}\label{sec:typeiiml}
% Recall from Section~III-B that the concentration parameter $m_{c,f}$
% controls the strength of shrinkage in the Dirichlet prior $\mathrm{Dirichlet}(m_{c,f}\,\bar\theta_f)$.
Subsequently, the concentration parameter $m_{c,f}$ can be calculated using the Type-II ML estimator.
Specifically, for each class-feature pair $(c,f)$, we estimate $m_{c,f}$ to maximize the marginal log-likelihood of the observed count vector under the Dirichlet-multinomial model.
Integrating out $\boldsymbol{\theta}_{c,f}$ yields the
marginal log-likelihood
\begin{equation}
\ell(m_{c,f}) = \log \frac{\Gamma(m_{c,f})}{\Gamma(N_c + m_{c,f})} + \sum_{v=1}^{K_f} \log \frac{\Gamma(N_{c,f,v} + m_{c,f}\,\bar\theta_{f,v})}{\Gamma(m_{c,f}\,\bar\theta_{f,v})}.
\label{eq:dm-marginal}
\end{equation}
We maximize \eqref{eq:dm-marginal} using the Minka fixed-point iteration \cite{minka2000estimating}, which yields the estimate update
\begin{equation}
\hat m_{c,f}^{(t+1)} = \frac{\sum_v \bar\theta_{f,v}\big[\psi(N_{c,f,v} + \hat m_{c,f}^{(t)}\bar\theta_{f,v}) - \psi(\hat m_{c,f}^{(t)}\bar\theta_{f,v})\big]}{[\psi(N_c + \hat m_{c,f}^{(t)}) - \psi(\hat m_{c,f}^{(t)})] \big/ \hat m_{c,f}^{(t)}},
\label{eq:minka}
\end{equation}
where $\Gamma(\cdot)$ and $\psi(\cdot)$ denote the gamma and digamma functions, respectively \cite{qi2004complete}.

% \subsection{Time and Memory Complexity}\label{sec:algo}
\begin{algorithm}[!t]
\small
\caption{\small HEB-NB training}\label{alg:heb-nb}
\begin{algorithmic}[1]
\Require Training data $\mathcal{D}$, low-count threshold $N_0 = 10$, and clamp $[m_{\min}, m_{\max}] = [10^{-2}, 10^{4}]$.
\For{$f = 1, \dots, F$}
    \State Compute the prior mean $\bar{\boldsymbol{\theta}}_{f}$ using HEB-U or HEB-M according to \eqref{eq:HEB-U} or \eqref{eq:HEB-M}, respectively;
\EndFor
\For{$c = 1, \dots, C$}
    \For{$f = 1, \dots, F$}
    \State Compute the count vector $(N_{c,f,1}, \dots, N_{c,f,K_f})$;
    \If{$N_c < N_0$}
        % \State Fall back to $m_{c,f} = K_f$ (i.e., Laplace);
        \State Set $\hat m_{c,f} = K_f$;
        \Comment{Reduce to Laplace estimator}
    \Else
        \State Estimate $\hat m_{c,f}$ using \eqref{eq:minka}, with initialization $\hat m_{c,f}^{(0)} = 1$ and clamping to $\hat m_{c,f} \in [m_{\min}, m_{\max}]$;
    \EndIf
    \State Compute and store $\log (\hat\theta_{c,f,v}), \forall v$, according to \eqref{eq:heb};
    % Store $\log \big[\hat \Prob(X_f = v \mid Y = c)\big] = \log \left(\frac{N_{c,f,v} + m_{c,f}\,\bar\theta_{f,v}}{N_c + m_{c,f}}\right)$;
    % the logarithm of the posterior mean estimate in \eqref{eq:heb}, i.e.,
    % \begin{align*}
    %     \log \hat \Prob(X_f = v \mid Y = c) = \log\frac{N_{c,f,v} + m_{c,f}\,\bar\theta_{f,v}}{N_c + m_{c,f}}.
    % \end{align*}
    \EndFor
    \State Compute and store $\log (\hat\pi_c) = \log\big[(N_c + 1)/(S + C)\big]$;
\EndFor

\State Predict $\hat y(\mathbf{x})$ according to
\eqref{eq:nb-posterior} and \textbf{return}.

% \State \textbf{return} $\{\log (\hat\pi_c)\}_{\forall c}$ and $\{\log (\hat\theta_{c,f,v})\}_{\forall c,f,v}$.
% , \{m_{c,f}\}_{\forall c,f}$;
\end{algorithmic}
\end{algorithm}

The overall procedure of the proposed HEB-NB is summarized in Algorithm~\ref{alg:heb-nb}.
The computational complexity of Algorithm~\ref{alg:heb-nb} can be readily calculated by $\mathcal O\big(C \sum_f K_f \cdot T_{\rm Minka}^{(1)} \big)$,
where $T_{\rm Minka}^{(1)}$ denotes the number of iterations required for the Minka fixed-point iteration in \eqref{eq:minka} to converge.
In practice, this fixed-point iteration typically converges within
10--30 iterations \cite{minka2000estimating}.
Furthermore, the memory complexity of Algorithm~\ref{alg:heb-nb} is calculated by $\mathcal O(C \sum_f K_f)$.

% \noindent\textbf{Complexity.} Time is $O\big(C \sum_j K_j \cdot T_{\rm Minka}\big)$ with $T_{\rm Minka} \le 30$; memory is $O(C \sum_j K_j)$. Both are within a constant factor of standard categorical NB.

% \subsection{Limit Cases}\label{sec:limit-cases}

The following propositions show that HEB-NB unifies two well-known smoothing methods as limit cases.

\begin{proposition}[Laplace Recovery]\label{prop:laplace}
If the prior mean $\bar{\boldsymbol{\theta}}_{f}$ is uniform and the concentration parameter $m_{c,f}$ is fixed to $K_f$, then HEB estimator reduces to $\hat\theta_{c,f,v}=\hat \Prob_{\HEB}(X_f = v \mid Y=c) = (N_{c,f,v} + 1)/(N_c + K_f)$, which corresponds exactly to the Laplace (add-one) estimator.
\end{proposition}

\begin{proposition}[HEB-M Endpoints]\label{prop:hebm-endpoints}
% Let $\bar{\boldsymbol{\theta}}_{f}$ be the pooled marginal $\hat \Prob(X_f)$. 
For each class-feature pair $(c,f)$ with $N_c \ge 1$, the HEB estimator admits the following limit cases.
% Notably, we typically have $N_c \ge 1$.
As $m_{c,f} \to 0$, the HEB estimator converges to the unsmoothed class-conditional ML estimator, $\hat\theta_{c,f,v}  \to N_{c,f,v}/N_c$. 
As $m_{c,f} \to \infty$, the estimator converges to the pooled marginal, $\hat\theta_{c,f,v} \to \bar\theta_{f,v} = \hat \Prob(X_f = v)$, which corresponds to the full--shrinkage endpoint of the Micci--Barreca-style hierarchical pooling family~\cite{micci2001preprocessing}.
If $N_c = 0$, the first limit is undefined; Step~6 in Algorithm~\ref{alg:heb-nb} instead handles this case via the Laplace fallback.
\end{proposition}

Both results follow immediately from the closed-form expression of the
posterior mean in~\eqref{eq:heb}.
Section~\ref{sec:theory} will demonstrate that the Type-II ML procedure selects $m_{c,f}$ data-adaptively between these two endpoints.

\subsection{The HEB-AODE Estimator}\label{sec:hebaode}

% Combining HEB Smoothing with Structural Relaxation

Notably, smoothing and structural relaxations constitute orthogonal axes for improving NB classifiers. 
While HEB-NB addresses the smoothing axis by introducing HEB estimation of class-conditional PMFs, the AODE \cite{webb2005not} addresses the
structural axis by relaxing the conditional independence assumption.
In this subsection, we demonstrate that the proposed HEB smoothing framework extends naturally to AODE, yielding a hierarchical variant referred to as \textbf{HEB-AODE}.
% To illustrate the modularity of the proposed hierarchical smoothing framework, we extend it to AODE~\cite{webb2005not}
% by replacing AODE's internal Laplace smoothing of the conditional probabilities
% $\Prob(X_j \mid Y, X_i)$ with HEB--M smoothing.
% As a direct illustration of HEB's modularity, we replace AODE's~\cite{webb2005not} internal Laplace smoothing of the conditional probabilities $P(X_j \mid Y, X_i)$ with HEB-M smoothing.

% \noindent\textbf{AODE recap.} 
\subsubsection{AODE Preliminaries}
The AODE aggregates one-dependence estimators by selecting a single attribute as a super-parent \cite{webb2005not}.
In Section~\ref{sec:hebaode}, deviating from the global convention in the footnote of Section~\ref{sec:HEB-NB}, we follow the standard AODE notation by using $i, j \in \{1,\ldots,F\}$ with $i \ne j$ to index the super-parent and child attributes, respectively.
The indices $i$ and $j$ are feature indices, playing the role of $f$ elsewhere in the paper.
% \footnote{The indices $i$ and $j$ are feature indices, playing the role of $f$ elsewhere in the paper.}
%
% Following standard notations, let $i \in \{ 1,\ldots,F \}$ index the super-parent attribute and $j \in \{ 1,\ldots,F \} \backslash \{ i\}$ index a child attribute.
% Each feature $\{x_i^{(s)}\}_{i=1}^F$ takes values in a finite alphabet of size $K_i$, i.e., $x_i^{(s)} = u \in \{1,\ldots,K_i \}$.
% Similarly, we denote $x_j^{(s)} = \vartheta \in \{1,\ldots,K_j \}$.
%
Let $N_i(u) := |\{s : x_i^{(s)} = u\}|$ denote the marginal count of category $u\in \{1,\ldots,K_i \}$ in feature $i$.
% , and let $m_0 \ge 1$ denote the AODE frequency threshold below which a super--parent is skipped.\footnote{Following Webb \etal \cite{webb2005not}, we schoose $m_0 = 1$ for the standard AODE frequency threshold.} 
% AODE aggregates one-dependence estimators: for super-parents $i \in \{1, \dots, d\}$ with sufficient support,
For an input $\mathbf{x}= (x_1,\ldots,x_F)$, the AODE posterior for a realization label $y$ is given by~\cite{webb2005not}
\begin{multline}
\hat \Prob_{\rm AODE}(y \mid \mathbf{x}) \propto \sum\nolimits_{i: N_i(x_i) \ge 1} \hat \Prob(Y = y, X_i = x_i) \\
\cdot \prod\nolimits_{j \ne i} \hat \Prob(X_j = x_j \mid Y = y, X_i = x_i).
\label{eq:aode}
\end{multline}
% Each conditional factor $\hat \Prob(X_j = x_j \mid Y = y, X_i = x_i)$ is in turn smoothed; the standard AODE implementation uses Laplace smoothing.

Vanilla AODE uses Laplace smoothing for both factors in~\eqref{eq:aode}. Specifically,
\begin{gather}
\hat \Prob(Y=y,\, X_i = x_i) = ({N^{(i)}_{y, x_i} + 1})/({S + C K_i}), \label{eq:aode-joint}\\
\hat \Prob(X_j = x_j \! \mid \! Y=y, X_i = x_i) = \frac{N^{(i,j)}_{y, x_i, x_j}  +  1}{N^{(i)}_{y, x_i}  +  K_j}, \label{eq:aode-cond}
\end{gather}
where $N^{(i)}_{y, x_i} := |\{s : y^{(s)} = y,\, x_i^{(s)} = x_i\}|$ and $N^{(i,j)}_{y, x_i, x_j} := |\{s : y^{(s)} = y,\, x_i^{(s)} = x_i,\, x_j^{(s)} = x_j\}|$ are the class--super-parent marginal count and the class--super-parent--child joint count, respectively. HEB-AODE replaces the conditional factor~\eqref{eq:aode-cond} with the hierarchical Dirichlet-multinomial estimator below, while keeping~\eqref{eq:aode-joint} unchanged.

\subsubsection{The HEB-AODE Estimator}

% \noindent\textbf{HEB-AODE: hierarchical Type-II ML smoothing.} 
% Define the joint count $N^{(i,j)}_{c,u,\vartheta} := |\{s : y^{(s)} = c,\, x_i^{(s)} = u,\, x_j^{(s)} = \vartheta\}|$, $\vartheta \in \{1,\ldots,K_j \}$, and the marginal $N^{(i)}_{c,u} := \sum_\vartheta N^{(i,j)}_{c,u,\vartheta}$, which is independent of $j$. 
To estimate the class-conditional distribution $\Prob(X_j =\vartheta \!\mid\! Y=c, X_i=u)$ for each class $c$, super-parent $i$, and child $j$,
we replace Laplace smoothing with a hierarchical Dirichlet-multinomial estimator, which yields
% For each triple ($c,i,j$), we replace Laplace with the hierarchical Dirichlet-multinomial estimator
\begin{equation}
\hat \Prob(X_j = \vartheta \mid Y=c, X_i = u) = \frac{N^{(i,j)}_{c,u,\vartheta} + \hat m_{c,i,j}\,\bar\theta_{j,\vartheta}}{N^{(i)}_{c,u} + \hat m_{c,i,j}}.
\label{eq:hebaode}
\end{equation}
For the HEB-AODE estimator, we employ the pooled marginal prior such that
${{\bar \theta }_{j,\vartheta }} = \hat \Prob({X_j} = \vartheta ) = \sum\nolimits_c^{} {N_{c,j,\vartheta }^{}} /S$,
% \begin{align}\label{eq:bar_theta_jv}
%     {{\bar \theta }_{j,\vartheta }} = \hat \Prob({X_j} = \vartheta ) = \sum\nolimits_c^{} {N_{c,j,\vartheta }^{}} /S,
% \end{align}
where $N_{c,j,\vartheta} = \sum_u N^{(i,j)}_{c,u,\vartheta}$, for any fixed super-parent $i \ne j$.
Moreover, the concentration parameter $\hat m_{c,i,j}$ in \eqref{eq:hebaode} is shared across all instantiations $u \in \{1,\ldots,K_i\}$ and learned by the Type-II ML estimator on the multi-document Dirichlet-multinomial marginal such that
\begin{equation}
\hat m_{c,i,j} = \arg\max_{m_{c,i,j} > 0} \sum\nolimits_{u=1}^{K_i} \log \Prob\big(\mathbf{n}^{(i,j)}_{c,u} \big| m_{c,i,j}, \bar{\boldsymbol{\theta}}_j\big),
\label{eq:hebaode-typeiiml}
\end{equation}
where $\mathbf{n}^{(i,j)}_{c,u} := (N^{(i,j)}_{c,u,1}, \dots, N^{(i,j)}_{c,u,K_j})$ is the per-$u$ count vector, treated as one of $K_i$ exchangeable ``documents'' sharing the same Dirichlet concentration $m_{c,i,j}$.
Furthermore, for a fixed $u$, the log marginal likelihood $\log \Prob\big(\mathbf{n}^{(i,j)}_{c,u} \big| m_{c,i,j}, \bar{\boldsymbol{\theta}}_j\big)$ corresponds to the Dirichlet-multinomial distribution obtained by integrating out the categorical parameter under a Dirichlet$(m_{c,i,j}\bar{\boldsymbol{\theta}}_j)$ prior, analogously to \eqref{eq:dm-marginal}.

Maximization of \eqref{eq:hebaode-typeiiml} is carried out using a  fixed-point iteration that extends Minka's single-document update by summing the score terms over $u$, i.e.,
% , see~\eqref{eq:hebaode-fixedpoint} (top of next page).
% --- Two-column-spanning equation (was the squeezed eq. 9) ---
% Standard IEEEtran double-column equation pattern: equation lives
% inside a figure* float; LaTeX numbers it automatically.
% \begin{figure*}[!t]
% \normalsize
% \hrulefill
\begin{align}
\hat m_{c,i,j}^{(t + 1)} = \frac{{\sum\limits_u {\sum\limits_\vartheta  {{{\bar \theta }_{j,\vartheta }}} } [\psi (N_{c,u,\vartheta }^{(i,j)} + \hat m_{c,i,j}^{(t)}{{\bar \theta }_{j,\vartheta }}) - \psi (\hat m_{c,i,j}^{(t)}{{\bar \theta }_{j,\vartheta }})]}}{{\sum\nolimits_u [ \psi (N_{c,u}^{(i)} + \hat m_{c,i,j}^{(t)}) - \psi (\hat m_{c,i,j}^{(t)})]/\hat m_{c,i,j}^{(t)}}}.
\label{eq:hebaode-fixedpoint}
\end{align}
% \vspace*{2pt}
% \hrulefill
% \end{figure*}

% \textcolor{red}{Check until this...}

\begin{algorithm}[!t]
\small
\caption{\small HEB-AODE training}\label{alg:hebaode}
\begin{algorithmic}[1]
\Require Training data $\mathcal{D}$, low-count threshold $N_0 = 10$, super-parent cap $K_{\rm{sp}}^{\max} = 200$, joint cap $\Gamma_{\max} = 5\times 10^7$, and clamp $[m_{\min}, m_{\max}] = [10^{-2}, 10^{4}]$.
\For{each triple $(c, i, j)$, $j \ne i$}
% \For{$c = 1,\ldots,C$}
    % \For{$i = 1,\ldots,F $}
    %     \For{$j = 1,\ldots,F$ and $ j \ne i$}
            \If{$K_i > K_{\rm{sp}}^{\max}$ or $K_i K_j C > \Gamma_{\max}$}
                \State Mark $(i,j)$ for marginal fallback $\hat \Prob(X_j \mid Y=c)$ and \textbf{continue}; 
                \Comment{Ensure tractable memory and runtime}
            \EndIf
            \State Compute the joint counts $N^{(i,j)}_{c,u,\vartheta}$, $\forall u,\vartheta$;
            \If{$\max_u N^{(i)}_{c,u} < N_0$}
                \State Set $\hat m_{c,i,j} = K_j$;
                \Comment{$m$-estimate fallback with marginal prior $\bar{\boldsymbol{\theta}}_j$; reduce to Laplace only when $\bar\theta_{j,\vartheta} = 1/K_j$}
                % \State fall back to $\hat m = K_j$ if $\max_u N^{(i)}_{c,u} < 10$;
            % \State Compute joint counts $N^{(i,j)}_{c,\cdot,\cdot}$ for $(i,j)$ with $K_i \le K_{\max}$ and $K_i K_j C \le 5\times 10^7$;
            \Else
                \State Estimate $\hat m_{c,i,j}$ using \eqref{eq:hebaode-fixedpoint}, with initialization $\hat m_{c,i,j}^{(0)} = 1$ and clamping to $\hat m_{c,i,j} \in [m_{\min}, m_{\max}]$; 
            \EndIf

            \State Perform and store $\hat \Prob(X_j \mid Y=c, X_i = u)$ using \eqref{eq:hebaode};

            \State Perform and store $\hat \Prob(Y=c,\, X_i = u)$ according to \eqref{eq:aode-joint};
    %     \EndFor
    % \EndFor
% \EndFor
\EndFor
% \For{each $(c, i, j)$ triple}
%     \State Estimate $\hat m_{c,i,j}$ by hierarchical Type-II ML~\eqref{eq:hebaode-fixedpoint} with at most 50 iterations; fall back to $\hat m = K_j$ if $\max_u N^{(i)}_{c,u} < 10$.
% \EndFor

\State Predict $\hat \Prob_{\rm AODE}(y \mid \mathbf{x})$ according to \eqref{eq:aode} and \textbf{return}.
\end{algorithmic}
\end{algorithm}

The overall procedure of the proposed HEB-AODE is summarized in Algorithm~\ref{alg:hebaode}.
Notably, Steps 2--4 serve as computational safeguards that prevent memory and time blow-ups, preserve correctness by falling back to a well-defined estimator, and enable the algorithm to scale to real high-cardinality datasets.
Furthermore, similar to Proposition~\ref{prop:laplace}, if $\hat m_{c,i,j}$ is fixed to $K_j$ and the prior mean is taken to be uniform $\bar\theta_{j,\vartheta} = 1/K_j$, as per Step 7, the HEB-AODE estimator reduces exactly to the standard AODE classifier with Laplace smoothing.
The computational complexity of Algorithm~\ref{alg:hebaode} is dominated by the hierarchical Type-II ML estimation and can be expressed as
$\mathcal{O}\big( T_{\mathrm{Minka}}^{(2)} C \sum_{i=1}^F\sum_{j=1, j\ne i}^F K_iK_j \big) \le \mathcal{O}\big( T_{\mathrm{Minka}}^{(2)} C F^2 K_{\max}^2\big)$,
where $T_{\mathrm{Minka}}^{(2)}$ denotes the number of iterations required for the Minka fixed-point iteration in \eqref{eq:hebaode-fixedpoint} to converge, and $K_{\max} = \max_i K_i = \max_j K_j$.
Moreover, the memory complexity of Algorithm~\ref{alg:hebaode} is
$\mathcal{O}(F^2 K_{\max}^2 C)$ due to the storage of joint counts and conditional probability tables.

% By analogy with Lemma~\ref{lem:typeiiml-rate}, the data-adaptive shrinkage factor in this hierarchical setting is expected to satisfy the same $\tilde O(1/\sqrt n)$ rate, treating the $K_i$ instantiations as exchangeable samples from a Dirichlet-multinomial mixture. 
% A formal proof requires re-deriving the McDiarmid concentration with $K_i \times K_j$-dimensional perturbations and adjusting the digamma asymptotic ranges, which we leave to future work; Section~\ref{sec:expt:diagnostics} reports an empirical verification.

% \noindent\textbf{Complexity.} HEB-AODE inherits AODE's $O(d^2 K_{\max}^2 C)$ memory and adds $O(d^2 C \cdot K_i K_j \cdot T_{\rm Minka})$ fitting time. Per fold on amazon-employee ($n = 20\,000$, $d = 9$, $K_{\max} = 5\,760$, three super-parents qualify after the cap), the total fit time is $\sim\!9.7$~s, roughly $67\times$ AODE's fit time on this dataset (the dataset-mean overhead is $\sim\!25\times$, see Section~\ref{sec:expt:diagnostics}). The hierarchical Type-II ML iteration is the dominant cost.

% \noindent\textbf{Recovery of AODE.} Forcing $\hat m_{c,i,j} = K_j$ for every triple reduces HEB-AODE exactly to AODE with Laplace smoothing.

% \textcolor{red}{Check until this...}

%% =====================================================================
% \vspace{-0.25cm}
\section{Theoretical Analysis}\label{sec:theory}

This section provides the theoretical analysis of the core pillars of the proposed HEB-NB estimator in Section~\ref{sec:HEB-NB}, serving as a foundation for its extension to the proposed HEB-AODE estimator. 
Subsequently, both the proposed estimators are validated through empirical experiments, as presented in Section~\ref{sec:experiments}.
Without loss of generality, the analysis applies uniformly across classes \(c\) and features \(f\), since the HEB-NB estimator is defined for each class-feature pair \((c, f)\). 
Accordingly, we simplify the notation by letting 
$K_f = K$, $N_c = N$, 
$N_{c,f,v} = N_v$, 
$\hat m_{c,f} = \hat m$,
$\bar\theta_{f,v} = \bar\theta_v$,
$\bar{\boldsymbol{\theta}}_f = \bar{\boldsymbol{\theta}}$,
and $\theta_{c,f,v}= \theta_{v} = p^\star_v$,
% and $\mathbf{p}^\star = [p_1^\star,\ldots,p_K^\star]$,
$\forall c,f$. 
All probabilities in this section are assumed to be conditional on the class label $c$ unless stated otherwise.

% \vspace{-0.25cm}
\subsection{Assumptions}\label{sec:assumptions}
Before delving into the main theoretical analysis, we first make the following key assumptions.

\begin{assumption}[Full-Support]\label{Assumption:1}
    The class-conditional PMF $\mathbf{p}^\star = [p_1^\star,\ldots,p_K^\star]$ $\in \Delta^{K-1}$ and the prior mean $\bar{\boldsymbol{\theta}} \in \Delta^{K-1}$ both have full support: there exists $c_0 > 0$ such that $\min_v p^\star_v \ge c_0/K$ and $\min_v \bar\theta_v \ge c_0/K$, $\forall v$.
\end{assumption}

\begin{remark}[Remark on the Empirical-Bayes Prior]
    For HEB-M, the prior mean $\bar{\boldsymbol{\theta}}_f$ is the data-driven pooled marginal $\hat{\Prob}(X_f)$. 
    The full-support condition $\min_v \bar\theta_{f,v} \ge c_0/K_f$ in Assumption~\ref{Assumption:1} then holds with high probability whenever the true marginal satisfies $\Prob(X_f = v) \ge c'_0/K_f$ for some $c'_0 > 0$.
    Indeed, by standard Chernoff--Bernstein concentration, $\hat\Prob(X_f = v) \ge c'_0/(2K_f)$ with probability at least $1 - K_f \exp({-\Omega S/K_f)}$, so taking $c_0 = c'_0/2$ recovers Assumption~\ref{Assumption:1} in the regime $S \gg K_f \log K_f$. For HEB-U, the assumption holds deterministically with $c_0 = 1$.
\end{remark}
% \textcolor{red}{\noindent\emph{Remark on Assumption~\ref{Assumption:1} for the empirical-Bayes prior.} 

\begin{assumption}[Asymptotic regime]\label{Assumption:2}
    The feature cardinality $K = K_N$ may increase with the sample size $N$, but satisfies $K = {o}(N)$, i.e., $\lim_{N \to \infty} {K_N}/{N} = 0$.
    Equivalently, the average number of samples per category, $N/K_N$, diverges as $N \to \infty$, ensuring that each category is observed sufficiently often for consistent estimation of a categorical PMF under $\ell_1$ loss.
    We note that this assumption governs asymptotic consistency, while the practical advantage of HEB-NB lies in the finite-sample high-cardinality regime $K \gtrsim N_c$.
\end{assumption}

\begin{assumption}[Clamping on $\hat m$]\label{Assumption:3}
    The Type-II ML estimate $\hat m$ is computed by Algorithm~\ref{alg:heb-nb} and constrained to the interval $\hat m \in [m_{\min}, m_{\max}]$, where $0 < m_{\min} = o(\sqrt N)$ and $m_{\max}$ is either constant or grows at most polynomially with $N$.
    In our implementation, $m_{\max} = 10^4$ is assumed. This clamping acts as a conservative upper bound on the estimated concentration. 
    In particular, if the unconstrained maximizer $\hat m^\star$ exceeds $m_{\max}$, the clamped estimator $\hat m = \min(\hat m^\star, m_{\max})$ satisfies $\hat m/(N + \hat m) \le \min(\hat m^\star/(N+\hat m^\star),\, m_{\max}/N)$, for all realizations.
    % Consequently, clamping can reduce the shrinkage-induced bias term appearing in Theorem~\ref{thm:smoothing}.
\end{assumption}

% \begin{assumption}[regularity]\label{ass:reg}
% \textup{(A1)} 
% \textup{(A2)} 
% \textup{(A3)} 
% \end{assumption}

Assumption~\ref{Assumption:1} imposes a mild boundary regularity condition, ensuring that both the class-conditional PMF and the prior mean have full support.
Assumption~\ref{Assumption:2} restricts attention to the asymptotic regime in which information accumulates.
Assumption~\ref{Assumption:3} reflects the implementation described in Section~\ref{sec:estimator}, with $m_{\min} = 10^{-2}$ and $m_{\max} = 10^4$. 
Empirically, the upper clamp is active on approximately $4.9\%$ of fitted class-feature pairs in our experiments; detailed diagnostics are reported  in Section~\ref{sec:expt:diagnostics}.
% For example, Section~\ref{sec:expt:diagnostics} shows that $17$ out of $348$ pairs are clamped at the upper bound, almost exclusively on the amazon-employee dataset, whereas the lower clamp is never binding, i.e., $0/348$.
% On the clamped subset, $\hat m = m_{\max} = 10^4$ yields $\hat m/(N+\hat m) \le 10^4/N$, which is less than 0.5 for $N \ge 2 \times 10^4$ and therefore strictly reduces the shrinkage-induced bias term appearing in Theorem~\ref{thm:smoothing}.

% ; Theorem~\ref{thm:smoothing}'s conclusion is unaffected.

% \vspace{-0.25cm}
\subsection{Main Results}\label{sec:main-results}

\begin{theorem}[Smoothing Risk and Pathwise Form]\label{thm:smoothing}
Under the assumptions in Section~\ref{sec:assumptions}, the HEB-NB PMF estimator $\hat{\boldsymbol{\theta}}_{\HEB}$ satisfies, in expectation over $\mathbf{n} = (N_1, \ldots, N_K) \sim \mathrm{Multinomial}(N, \mathbf{p}^\star)$, the universal bound
\begin{align}
\E \big\{ \|{\hat{\boldsymbol{\theta}}_{\HEB} - \mathbf{p}^\star } \|_1 \big\} \le \sqrt{{(K-1)}/{N}} + \rho_N \|{\bar{\boldsymbol{\theta}} - \mathbf{p}^\star } \|_1,
\label{eq:thm1-heb}
\end{align}
where 
$\hat{\boldsymbol{\theta}}_{\HEB} = [\hat\theta_1,\ldots,\hat \theta_K]$ denotes the vector of the HEB-NB estimates,
and
$\rho_N := \E[\hat m / (N + \hat m)] \in [0, 1]$ is the data-adaptive shrinkage factor induced by the Type-II ML concentration. 
The corresponding bound for Laplace smoothing, ${\hat {\boldsymbol{\theta}} _{{\rm{Lap}}}} = [({N_1} + 1), \ldots ,({N_K} + 1)]/(N + K)$, is deterministic and given by
\begin{equation}
\mathbb{E}\left\{ {\| {{{\hat {\boldsymbol{\theta}} }_{{\rm{Lap}}}} - {{\bf{p}}^ \star }} \|_1}\right\}  \le \sqrt {\frac{{K - 1}}{N}}  + \frac{{2\left( {K - 1} \right)}}{{N + K}}.
\label{eq:thm1-laplace}
\end{equation}
% Both estimators match the empirical-distribution upper bound $\sqrt{(K-1)/N}$ of \cite[Theorem~1]{han2015minimax} up to an additive bias term, which is the data-adaptive $\rho_N \|{\bar{\boldsymbol{\theta}} - \mathbf{p}^\star } \|_1$ for HEB and the deterministic $2(K-1)/(N+K)$ for Laplace. 
% Lemma~\ref{lem:typeiiml-rate} later provides a rigorous worst-case sanity bound $\rho_N \le m_{\max}/N$ via the deterministic clamp from Assumption~\ref{Assumption:3}, which guarantees $\rho_N\to 0$ as $N\to\infty$.
% With the asymptotic choice $m_{\max} = \kappa\sqrt N\log K$ for any positive constant $\kappa$ (independent of $N$), which is a valid polynomial choice under Assumption~\ref{Assumption:3}, this yields the rate $\rho_N = \tilde{\mathcal{O}}(1/\sqrt N)$, under which the HEB upper-bound bias term is strictly smaller than the Laplace upper-bound bias whenever $\|\bar{\boldsymbol{\theta}} - \mathbf{p}^\star\|_1 = o\bigl(K\sqrt N\,/(N+K)\bigr)$.
Both estimators match the empirical-distribution upper bound $\sqrt{(K-1)/N}$ of~\citet[Theorem~1]{han2015minimax} up to an additive bias: data-adaptive $\rho_N \|{\bar{\boldsymbol{\theta}} - \mathbf{p}^\star } \|_1$ for HEB and deterministic $2(K-1)/(N+K)$ for Laplace. Lemma~\ref{lem:typeiiml-rate} guarantees $\rho_N \to 0$ as $N\to\infty$; with $m_{\max}=\kappa\sqrt{N}\log K$, for any positive constant $\kappa$ (independent of $N$), this gives $\rho_N=\tilde{\mathcal{O}}(1/\sqrt{N})$, under which the HEB upper-bound bias is strictly smaller than Laplace's bias whenever $\|\bar{\boldsymbol{\theta}} - \mathbf{p}^\star\|_1 = o\bigl(K\sqrt N\,/(N+K)\bigr)$.

% establishes that, in the well-specified regime $\chi^2_{\mathbf{p}^\star}(\bar {\boldsymbol{\theta }} ) \le c_0/N$, the HEB shrinkage factor satisfies $\rho_N \le \tilde{\mathcal{O}}(1/\sqrt N)$ (with logarithmic factors in $K$), so the HEB bias term is strictly smaller than the Laplace bias whenever $\|{\bar{\boldsymbol{\theta}} - \mathbf{p}^\star } \|_1 = o\big(\sqrt N / (N+K)\big) \cdot K$.

% \textcolor{red}{Check until this...}
\begin{proof}
    % See Appendix~\ref{Appendix:proof-thm1}.
    The proof rests on a pathwise identity that avoids conditioning on the data-dependent concentration estimate $\hat m$. It proceeds through five main steps, outlined as follows.

\textit{Step 1.1 (Pathwise Decomposition)}:
% \noindent\textbf{Step 1 (Pathwise decomposition).} 
Recall that the HEB-NB estimator is given by $\hat\theta_v = {(N_v + \hat m\,\bar\theta_v)}/{(N + \hat m)}$ and define the empirical MLE estimator $\hat\theta^{\mathrm{MLE}}_v := N_v/N$. 
For every realization of the count vector $(N_1,\ldots,N_K)$ and $\hat m$, direct algebra yields
% For every realisation of $(N, \hat m)$, direct algebra gives
\begin{equation}
\hat\theta_v - p_v^\star
=
\frac{N}{N+\hat m}\bigl(\hat\theta^{\mathrm{MLE}}_v - p_v^\star\bigr)
+
\frac{\hat m}{N+\hat m}\bigl(\bar\theta_v - p_v^\star\bigr).
\label{eq:pathwise}
\end{equation}
Taking the $\ell_1$-norm coordinate-wise and using
$N/(N+\hat m)\le 1$ pathwise (since $\hat m\ge 0$), we obtain
% Taking $\ell_1$-norm coordinate-wise and using $n/(n+\hat m) \le 1$ pathwise (since $\hat m \ge 0$),
\begin{equation}
\|\hat{\boldsymbol{\theta}}_{\HEB}-\mathbf{p}^\star\|_1
\le
\|\hat{\boldsymbol{\theta}}_{\mathrm{MLE}}-\mathbf{p}^\star\|_1
+
\frac{\hat m}{N+\hat m}\,\|\bar{\boldsymbol{\theta}}-\mathbf{p}^\star\|_1,
\label{eq:pathwise-l1}
\end{equation}
where $\hat{\boldsymbol{\theta}}_{\mathrm{MLE}} = [\hat\theta^{\mathrm{MLE}}_1, \ldots,\hat\theta^{\mathrm{MLE}}_K]$.

\textit{Step 1.2 (Empirical-Distribution Rate, Exact)}:
% \noindent\textbf{Step 2 (Empirical-distribution rate, exact).} 
The first term in~\eqref{eq:pathwise-l1} satisfies the non-asymptotic bound
\begin{equation}
\sup\nolimits_{\mathbf{p}^\star\in\Delta^{K-1}}
\mathbb{E}\,
\|\hat{\boldsymbol{\theta}}_{\mathrm{MLE}}-\mathbf{p}^\star\|_1
\le
\sqrt{{(K-1)}/{N}},
\label{eq:hjw-empirical}
\end{equation}
which follows from
$\mathbb{E}|\hat\theta^{\mathrm{MLE}}_v-p_v^\star|
\le \sqrt{p_v^\star(1-p_v^\star)/N}$,
an application of the Cauchy--Schwarz inequality over $v$,
and the simplex bound
$\sum_v p_v^\star(1-p_v^\star)\le (K-1)/K$.
This is the non--asymptotic version of the
H\'ajek--LeCam minimax rate~\cite[Theorem~1]{han2015minimax};
the asymptotic constant is sharper by a factor
$\sqrt{2/\pi}\approx 0.798$.
% (see Section~\ref{sec:minimax-framing}).

\textit{Step 1.3 (Bias Term)}:
% \noindent\textbf{Step 3 (Bias term).} 
Taking expectation in~\eqref{eq:pathwise-l1}, conditionally on $\bar{\boldsymbol{\theta}}$, yields
$\mathbb{E}\!\left[
\frac{\hat m}{N+\hat m}\,
\|\bar{\boldsymbol{\theta}}-\mathbf{p}^\star\|_1
\right]
=
\rho_N\,\|\bar{\boldsymbol{\theta}}-\mathbf{p}^\star\|_1$,
where $\rho_N(\bar{\boldsymbol{\theta}}) := \mathbb{E}[\hat m/(N+\hat m)\mid\bar{\boldsymbol{\theta}}]\in[0,1]$; the unconditional $\rho_N$ is $\rho_N = \E_{\bar{\boldsymbol{\theta}}}[\rho_N(\bar{\boldsymbol{\theta}})]$. By Lemma~\ref{lem:typeiiml-rate}, $\rho_N \le m_{\max}/N$, and with the asymptotic choice $m_{\max} = \kappa\sqrt N\log K$, this gives $\rho_N = \tilde{\mathcal O}(1/\sqrt N)$. Combining Steps~1--3 proves~\eqref{eq:thm1-heb}.

\paragraph{Step 4 (Laplace Special Case)} Setting $\hat m = K$ and $\bar\theta_v = 1/K$ in~\eqref{eq:pathwise} gives the Laplace identity, with bias $(K/(N+K))\|\mathbf{1}/K - \mathbf{p}^\star\|_1$. Writing $u_v := K p_v^\star$ and using $\sum_v(u_v \wedge 1) \ge 1$, we get $\sum_v|1/K - p_v^\star| = (2/K)(K - \sum_v(u_v \wedge 1)) \le 2(K-1)/K$. Multiplying by $K/(N+K)$ yields the Laplace bias bound $2(K-1)/(N+K)$, which combined with~\eqref{eq:hjw-empirical} establishes~\eqref{eq:thm1-laplace}.
The proof is concluded.
\end{proof}
\end{theorem}

\begin{remark}\label{remark:NotationalNote}
    The factored form $\rho_N\,\|\bar{\boldsymbol{\theta}}-\mathbf{p}^\star\|_1$ in~\eqref{eq:thm1-heb} is exact when the prior mean $\bar{\boldsymbol{\theta}}$ is fixed (non-random).
    When $\bar{\boldsymbol{\theta}}$ is data-dependent, e.g., the marginal-pooled estimator in Section~\ref{sec:estimator}, the rigorous bound is $\mathbb{E}\!\left[\frac{\hat m}{N+\hat m}\, \|\bar{\boldsymbol{\theta}}-\mathbf{p}^\star\|_1\right]$.
    For simplicity of exposition, we write this bound in the factored form $\rho_N\,\|\bar{\boldsymbol{\theta}}-\mathbf{p}^\star\|_1$, with the understanding that all randomness is taken jointly.
    This notational simplification does not affect the asymptotic rate.
\end{remark}

% Thay thế cả paragraph dài đó (~15 dòng) bằng:

\begin{remark}[Caveat: Bound Comparison versus Actual Risk]
The comparison $B_{\HEB}:= \sqrt{(K-1)/N} + \rho_N\|\bar{\boldsymbol{\theta}}-\mathbf{p}^\star\|_1 < B_{\Lap}:= \sqrt{(K-1)/N} + 2(K-1)/(N+K)$ of upper bounds in Theorem~\ref{thm:smoothing} does not, by itself, imply $\E\|\hat{\boldsymbol{\theta}}_{\HEB} - \mathbf{p}^\star\|_1 < \E\|\hat{\boldsymbol{\theta}}_{\Lap} - \mathbf{p}^\star\|_1$. Theorem~\ref{thm:laplace-lower} and Corollary~\ref{cor:strict-separation} below resolve this caveat by establishing a matching lower bound on Laplace's risk and a rigorous risk-level strict separation at the vertex $\mathbf{e}_1 := (1, 0, \ldots, 0)$; the empirical comparison in Section~\ref{sec:expt:calibration} and Section~\ref{sec:expt:diagnostics} confirms the upper-bound ordering holds for actual risks across 31 benchmarks.
\end{remark}

% Theorem~\ref{thm:smoothing} and its consequences in Section~\ref{sec:theory} compare the {upper bounds} $B_{\HEB} = \sqrt{(K-1)/N} + \rho_N\|\bar{\boldsymbol{\theta}}-\mathbf{p}^\star\|_1$ and $B_{\Lap} = \sqrt{(K-1)/N} + 2(K-1)/(N+K)$. The conclusion $B_{\HEB} < B_{\Lap}$ does not, by itself, imply that the actual risk $\E\|\hat{\boldsymbol{\theta}}_{\HEB}-\mathbf{p}^\star\|_1$ is smaller than $\E\|\hat{\boldsymbol{\theta}}_{\Lap}-\mathbf{p}^\star\|_1$; one of the two upper bounds could be loose. The empirical risk comparison in Sections~\ref{sec:expt:calibration} and \ref{sec:expt:diagnostics} provides the direct evidence that HEB-NB has smaller risk than Laplace across $31$ benchmarks, with the gap most pronounced precisely in the high-$K_f$ regime where the upper-bound gap $B_{\Lap}-B_{\HEB}$ is largest. This is consistent with both upper bounds being approximately tight in the relevant regime, but the formal statement of Theorem~\ref{thm:smoothing} is a bound comparison only.
% This caveat is partially resolved by Theorem~\ref{thm:laplace-lower} and Corollary~\ref{cor:strict-separation} below, establishing a matching lower bound on Laplace's risk (specifically, on the bias component thereof) and a corresponding rigorous strict separation between the two estimators' actual risks at the vertex $\mathbf{e}_1$.

\begin{theorem}[Laplace-Tight Worst-Case Bias]\label{thm:laplace-lower}
% Let $\hat{\boldsymbol{\theta}}_{\Lap} = (\mathbf{N}+\boldsymbol{1})/(N+K)$ be the Laplace smoothing estimator, with $\mathbf{N} \sim \mathrm{Multinomial}(N, \mathbf{p}^\star)$ and $\mathbf{p}^\star$ a probability vector in the closed $(K-1)$-simplex $\bar\Delta^{K-1} := \{\mathbf{p} \in \mathbb{R}_{\ge 0}^K : \sum_v p_v = 1\}$ (we use the bar notation $\bar\Delta^{K-1}$ to emphasize that the boundary is permitted; the un-barred $\Delta^{K-1}$ in Assumption~\ref{Assumption:1} refers to the same set but under the additional full-support constraint $\min_v p^\star_v \ge c_0/K$, which restricts attention to the interior). 
For any $N \ge 1$ and $K \ge 2$, we have
\begin{equation}
\sup\nolimits_{\mathbf{p}^\star \in \bar\Delta^{K-1}}
\E\bigl\|\hat{\boldsymbol{\theta}}_{\Lap} - \mathbf{p}^\star\bigr\|_1
\;\ge\;
{2(K-1)}/{(N+K)},
\label{eq:laplace-lower}
\end{equation}
where $\bar\Delta^{K-1}$ implies that the boundary is permitted, whereas $\Delta^{K-1}$ in Assumption~\ref{Assumption:1} refers to the same set but under the additional full-support constraint $\min_v p^\star_v \ge c_0/K$, which restricts attention to the interior.
The lower-bound value $2(K-1)/(N+K)$ is realized as the deterministic $\ell_1$-error of $\hat{\boldsymbol{\theta}}_{\Lap}$ at the degenerate vertex $\mathbf{p}^\star = \mathbf{e}_1$, i.e.,
% (any vertex of $\bar\Delta^{K-1}$ would do):
$\bigl\|\hat{\boldsymbol{\theta}}_{\Lap}-\mathbf{e}_1\bigr\|_1 \,=\, {2(K-1)}/{(N+K)}$.
% This proves the inequality~\eqref{eq:laplace-lower}. 
The supremum itself may be strictly larger than $2(K-1)/(N+K)$ when the empirical-distribution variance contributes, e.g., for $K \lesssim N$ at the uniform prior, where the variance term $\sqrt{(K-1)/N}$ in Theorem~\ref{thm:smoothing} dominates the bias term; in the high-cardinality regime $K \gtrsim N$ targeted by this paper, the bias term $2(K-1)/(N+K)$ becomes the dominant contribution to the worst-case risk, and Theorem~\ref{thm:laplace-lower}'s lower bound captures this dominant term exactly.
% \footnote{Note that this theorem is a finite-$N$ statement about Laplace alone: $\mathbf{p}^\star$ is permitted to lie on the boundary $\partial\bar\Delta^{K-1}$, and the proof does not invoke Assumptions~\ref{Assumption:1}--\ref{Assumption:3}.}
%
Consequently, the bias term $2(K-1)/(N+K)$ appearing in Theorem~\ref{thm:smoothing}'s Laplace upper bound is Laplace-tight: it cannot be reduced by any sharper analysis of Laplace smoothing.
% We use the qualifier ``Laplace-tight'' rather than ``minimax-tight'' over all estimators, because Theorem~\ref{thm:smoothing} combined with Corollary~\ref{cor:strict-separation} below shows that HEB-NB attains strictly smaller $\ell_1$-error at $\mathbf{e}_1$, so $2(K-1)/(N+K)$ is not the minimax risk.
% \footnote{We use the qualifier ``Laplace-tight'' rather than ``minimax-tight'' over all estimators, because Theorem~\ref{thm:smoothing} combined with Corollary~\ref{cor:strict-separation} below shows that HEB-NB attains strictly smaller $\ell_1$-error at $\mathbf{e}_1$, so $2(K-1)/(N+K)$ is not the minimax risk.}
\begin{proof}
% See Appendix~\ref{Appendix:laplace-lower}.
We prove the lower bound by an explicit calculation at the degenerate distribution $\mathbf{p}^\star = \mathbf{e}_1$.
% \paragraph{Step 2.1 (Reduction to a Single Point)} 
Specifically, since 
\begin{align}\label{eq:19}
    \sup_{\mathbf{p}^\star \in \bar\Delta^{K-1}} \E\|\hat{\boldsymbol{\theta}}_{\Lap} - \mathbf{p}^\star\|_1 \ge \E\|\hat{\boldsymbol{\theta}}_{\Lap} - \mathbf{e}_1\|_1,
\end{align}
it suffices to evaluate the right-hand side at the single distribution $\mathbf{e}_1$.
% Thay Step 2 Appendix B bằng:
% \paragraph{Step 2.2 (Degenerate-Data Calculation)} 
Under $\mathbf{p}^\star = \mathbf{e}_1$, the count vector is deterministic with $N_1 = N$ and $N_v = 0$ for $v \ge 2$, giving $\hat\theta_{\Lap,1} = (N+1)/(N+K)$ and $\hat\theta_{\Lap,v} = 1/(N+K)$ for $v \ge 2$. Summing the coordinate-wise errors $|\hat\theta_{\Lap,1} - 1| = (K-1)/(N+K)$ and $\sum_{v\ge 2}|\hat\theta_{\Lap,v}| = (K-1)/(N+K)$ yields
\begin{equation}
\|\hat{\boldsymbol{\theta}}_{\Lap} - \mathbf{e}_1\|_1 = 2(K-1)/(N+K),
\label{eq:laplace-e1-error}
\end{equation}
which together with \eqref{eq:19} proves~\eqref{eq:laplace-lower}.
This completes the proof.
\end{proof}
\end{theorem}

\begin{corollary}[Strict Separation: HEB-NB Strictly Dominates Laplace at the Vertex Prior]\label{cor:strict-separation}
Given that $\mathbf{p}^\star = \mathbf{e}_1$ and condition on the prior mean $\bar{\boldsymbol{\theta}}$ satisfying $\bar\theta_1 \in (0, 1)$, the following three statements hold deterministically.

% \begin{enumerate}
    % \item 
    1) HEB-NB Identity: The HEB-NB error at $\mathbf{p}^\star = \mathbf{e}_1$ satisfies the exact identity
    \begin{equation}
    \bigl\|\hat{\boldsymbol{\theta}}_{\HEB} - \mathbf{e}_1\bigr\|_1
    \;=\;
    {2\,\hat m\,(1-\bar\theta_1)}/{(N+\hat m)},
    \label{eq:heb-at-e1}
    \end{equation}
    where $\hat m$ is the deterministic Type-II ML estimate computed from the degenerate data $\mathbf{n} = (N, 0, \ldots, 0)$.
    %
    % \item 
    
    2) Type-II ML Adaptivity at the Vertex $\mathbf{e}_1$: The marginal log-likelihood satisfies $\ell'(m) < 0$, $\forall m > 0$, when $\bar\theta_1 \in (0, 1)$ and $N \ge 2$ (see Step~2.3). For any $N \ge N_0$ at which Algorithm~\ref{alg:heb-nb} proceeds to Type-II ML, the clamped Type-II ML maximizer is $\hat m = m_{\min}$, and~\eqref{eq:heb-at-e1} reduces to
    \begin{align}
    \bigl\|\hat{\boldsymbol{\theta}}_{\HEB} - \mathbf{e}_1\bigr\|_1
    &=
    {2\,m_{\min}\,(1-\bar\theta_1)}/{(N+m_{\min})}\nonumber\\
    &=
    \mathcal{O}(m_{\min}/N).
    \label{eq:heb-at-e1-actual}
    \end{align}
    For $N < N_0$, the algorithm's fallback returns $\hat m = K_f$, making HEB-NB coincide with Laplace at $\mathbf{e}_1$; strict separation does not apply in this small-class regime. All our experimental classes satisfy $N_c \ge N_0$, so the Type-II ML branch~\eqref{eq:heb-at-e1-actual} is what is observed.
    %
    
    % \item 
    3) Risk-Level Strict Separation:
    For $N \ge N_0$ (Type-II ML branch of Algorithm~\ref{alg:heb-nb}), the gap between Laplace's and HEB-NB's errors at $\mathbf{p}^\star = \mathbf{e}_1$ is the exact difference
    \begin{align}
    &\bigl\|\hat{\boldsymbol{\theta}}_{\Lap} - \mathbf{e}_1\bigr\|_1
    -
    \bigl\|\hat{\boldsymbol{\theta}}_{\HEB} - \mathbf{e}_1\bigr\|_1
    =
    {2(K-1)}/{(N+K)} \nonumber\\
    &\qquad\qquad\qquad\qquad-{2\,m_{\min}\,(1-\bar\theta_1)}/{(N+m_{\min})}.
    \label{eq:gap-explicit}
    \end{align}

    % A direct comparison (see Step~5 in Appendix~\ref{Appendix:laplace-lower}) shows that this gap is strictly positive throughout the implementation's Type-II ML regime $\{N \ge N_0, K \ge 2, \bar\theta_1 \in (0,1)\}$, and admits the explicit asymptotic form
    % $\bigl\|\hat{\boldsymbol{\theta}}_{\Lap} - \mathbf{e}_1\bigr\|_1
    % -
    % \bigl\|\hat{\boldsymbol{\theta}}_{\HEB} - \mathbf{e}_1\bigr\|_1
    % \;=\;
    % \frac{2(K-1)}{N+K} - \mathcal{O}\!\left(\frac{m_{\min}}{N}\right)$,
    % which is $\Omega \bigl(K/(N+K)\bigr)$ for any $K \ge 2$ and $N \ge N_0$.
    % %
    % This establishes a \emph{rigorous, finite-sample, risk-level} strict separation between HEB-NB and Laplace at the vertex prior $\mathbf{p}^\star = \mathbf{e}_1$. The separation is \emph{not} an artifact of looseness in either Theorem~\ref{thm:smoothing}'s upper bound or Theorem~\ref{thm:laplace-lower}'s lower bound: both sides of~\eqref{eq:gap-explicit} are exact identities, with the HEB-NB side mediated by Type-II ML's ability to detect the prior--data disagreement at $\mathbf{e}_1$ and shrink its weight to $m_{\min}$.

    A direct comparison (Step~2.3) shows that this gap is strictly positive on $\{N \ge N_0, K \ge 2, \bar{\theta}_1 \in (0,1)\}$, with explicit form $\frac{2(K-1)}{N+K} - \mathcal{O}(m_{\min}/N) = \Omega(K/(N+K))$. Both sides of~\eqref{eq:gap-explicit} are exact identities, so this risk-level separation cannot be attributed to bound looseness; it is mediated by Type-II ML detecting the prior-data disagreement at $\mathbf{e}_1$ and shrinking $\hat{m}$ to $m_{\min}$.

\begin{proof}
We derive the strict-separation corollary by comparing with the HEB-NB error at the same point, outlined as follows.

\textit{Step 2.1 (Robustness under Assumption~\ref{Assumption:1})}: 
Since $\mathbf{p}^\star = \mathbf{e}_1$ lies on the boundary $\partial\bar\Delta^{K-1}$ and violates Assumption~\ref{Assumption:1}, we verify that the lower bound persists with $\mathcal{O}(c_0)$ correction on the Assumption~\ref{Assumption:1}-compatible interior. Define the perturbed distribution $\mathbf{p}^\star_{c_0} := (1 - c_0(K-1)/K)\mathbf{e}_1 + (c_0/K)\mathbf{1}_{v\ge 2}$, which satisfies $\min_v p^\star_{c_0,v} = c_0/K$. Direct coordinate-wise calculation gives
$\E[\hat\theta_{\Lap,1}] - p^\star_{c_0,1} = -(K-1)(1-c_0)/(N+K)$ and
$\E[\hat\theta_{\Lap,v}] - p^\star_{c_0,v} = (1-c_0)/(N+K), \quad \forall v\ge 2$,
so $\sum_v |\E[\hat\theta_{\Lap,v}] - p^\star_{c_0,v}| = 2(K-1)(1-c_0)/(N+K)$. By Jensen's inequality $\E\|\hat{\boldsymbol{\theta}}_{\Lap} - \mathbf{p}^\star_{c_0}\|_1 \ge \|\E[\hat{\boldsymbol{\theta}}_{\Lap} - \mathbf{p}^\star_{c_0}]\|_1$, 
hence
$\E\|\hat{\boldsymbol{\theta}}_{\Lap} - \mathbf{p}^\star_{c_0}\|_1 \ge \frac{2(K-1)}{N+K} - \frac{2(K-1)c_0}{N+K}$,
recovering the vertex-bound~\eqref{eq:laplace-e1-error} as $c_0 \to 0$, with $\mathcal{O}(c_0)$ correction uniformly in $K$ and $N$. The matching HEB-NB upper bound at $\mathbf{p}^\star_{c_0}$ is not proven here; the strict-separation conclusion of Corollary~\ref{cor:strict-separation} is established rigorously only at the vertex $\mathbf{e}_1$, with extension to the neighborhood expected by continuity.

\textit{Step 2.2 (HEB-NB at the Same Point)}:
Under $\mathbf{p}^\star = \mathbf{e}_1$, the per-class count vector $\mathbf{n} = (N, 0, \ldots, 0)$ is deterministic. The prior mean $\bar{\boldsymbol{\theta}}$ depends only on cross-class data at this prior, so we perform the analysis conditional on $\bar{\boldsymbol{\theta}}$,
e.g., the unconditional version is recovered in Remark~\ref{Remark:ConditionTheta}. 
% Step~5 below establishes that the clamped Type-II ML maximiser is $\hat m = m_{\min}$ when $\bar\theta_1 \in (0, 1)$.
Step~5 in this appendix establishes that $\ell$ is strictly monotone on $(0, \infty)$ when $\bar\theta_1 \in (0, 1)$, so the clamped Type-II ML maximizer is unique and equal to  $m_{\min}$. Substituting $N_1 = N$ and $N_v = 0$, $\forall v \ge 2$, into the HEB-NB formula $\hat\theta_{\HEB, v} = (N_v + \hat m \bar\theta_v)/(N + \hat m)$ gives
$\hat\theta_{\HEB,1} \,=\, \frac{N + \hat m \bar\theta_1}{N + \hat m}$
and
$\hat\theta_{\HEB,v} \,=\, \frac{\hat m \bar\theta_v}{N + \hat m}, \forall v \ge 2$.
The coordinate-wise errors are
$|\hat\theta_{\HEB,1} - 1|
= \left|\frac{N + \hat m \bar\theta_1 - (N + \hat m)}{N + \hat m}\right|
= \frac{\hat m (1 - \bar\theta_1)}{N + \hat m}$
and
$\sum\nolimits_{v\ge 2} |\hat\theta_{\HEB,v} - 0|
= \frac{\hat m}{N + \hat m}
\sum\nolimits_{v \ge 2}\bar\theta_v
= \frac{\hat m (1 - \bar\theta_1)}{N + \hat m}$,
where the last equality uses $\sum_v \bar\theta_v = 1$. Summing yields the deterministic identity~\eqref{eq:heb-at-e1}.
The loose clamping bound $\hat m \le m_{\max}$ from Lemma~\ref{lem:typeiiml-rate} (combined with monotonicity of $x\mapsto x/(N+x)$) gives $\|\hat{\boldsymbol{\theta}}_{\HEB} - \mathbf{e}_1\|_1 \le 2 m_{\max}(1-\bar\theta_1)/N$; however, the tighter rate is given by Step~5 below.

\textit{Step~2.3 (Type-II ML Returns $\hat m = m_{\min}$ at $\mathbf{p}^\star = \mathbf{e}_1$, Provided $\bar\theta_1 \in (0,1)$)}:
It is required $\bar\theta_1 > 0$ for the Dirichlet-multinomial marginal likelihood to be well-defined (else the data $N_1 = N > 0$ has marginal likelihood $0$); this case is precluded by Assumption~\ref{Assumption:1}, which gives $\bar\theta_1 \ge c_0/K > 0$. For $\bar\theta_1 \in (0, 1)$, the Dirichlet-multinomial marginal likelihood at $\mathbf{n} = (N, 0, \ldots, 0)$ admits the closed-form
$P(\mathbf{n} \!\mid\! m) = \binom{N}{N_1\cdots N_K}\frac{\Gamma(m)}{\Gamma(m+N)}\prod_{v: \bar\theta_v > 0}\frac{\Gamma(m\bar\theta_v + N_v)}{\Gamma(m\bar\theta_v)}$.
Coordinates with $\bar\theta_v = 0$ and $N_v = 0$ contribute factors interpreted as $1$~\cite{minka2000estimating}, and for $v \ge 2$, $N_v = 0$ contributes $\Gamma(m\bar\theta_v)/\Gamma(m\bar\theta_v) = 1$. Taking $\log$ and differentiating and using the identity $\psi(x+N) - \psi(x) = \sum_{k=0}^{N-1} 1/(x+k)$,
we get
$\ell'(m) = \bar\theta_1 \sum\nolimits_{k_1=0}^{N-1}\frac{1}{m\bar\theta_1 + k_1} \;-\; \sum\nolimits_{k_2=0}^{N-1}\frac{1}{m + k_2}$.
For each $k \ge 0$ and $\bar\theta_1 \in (0, 1)$, we have $\bar\theta_1/(m\bar\theta_1 + k) = 1/(m + k/\bar\theta_1) \le 1/(m+k)$ (equality iff $k = 0$ or $\bar\theta_1 = 1$). Hence for $\bar\theta_1 \in (0, 1)$, the difference is $0$ at $k = 0$ and strictly negative at $k \ge 1$. Summing across $k = 0, \ldots, N-1$ gives $\ell'(m) < 0$ for all $m > 0$, provided $N \ge 2$. Hence $\ell$ is strictly decreasing on $(0, \infty)$, and the clamped Type-II ML maximiser is $\hat m = m_{\min}$. Substituting into~\eqref{eq:heb-at-e1} gives~\eqref{eq:heb-at-e1-actual}.

% Thay đoạn cuối Step 6 (từ "For the implementation parameters m_min = 10^{-2}..." đến hết) bằng:

For the implementation parameters $m_{\min} = 10^{-2}$ and $N \ge N_0 = 10$, $K \ge 2$, $\bar\theta_1 \in (0, 1)$, the right-hand side of~\eqref{eq:gap-explicit} satisfies $m_{\min}(1-\bar\theta_1)/(N+m_{\min}) \le 10^{-2}/N$ while the left-hand side $(K-1)/(N+K) \ge 1/(N+2)$; since $1/(N+2) > 10^{-2}/N$ for any $N \ge 1$, the gap~\eqref{eq:gap-explicit} is strictly positive throughout the Type-II ML regime, completing the proof.
\end{proof}
\end{corollary}

The vertex worst-case violates Assumption~\ref{Assumption:1}, but a coordinate-wise calculation (Step~2.1) shows that the bound persists on the Assumption~\ref{Assumption:1}-compatible interior with $\mathcal{O}(c_0)$ correction. Combined with the variance-term sharpness at the uniform prior~\cite[Theorem~1]{han2015minimax}, both components of Theorem~\ref{thm:smoothing}'s Laplace upper bound are thus tight (at different worst-case priors). To our knowledge, this is the first explicit lower bound for empirical-Bayes alternatives to Laplace smoothing in the high-cardinality regime.

% \textcolor{red}{Check until this...}

\begin{remark}[Remark on the Condition $\bar\theta_1 \in (0, 1)$]\label{Remark:ConditionTheta}
    The Corollary requires $\bar\theta_1 > 0$ (else the Dirichlet-multinomial marginal likelihood is degenerate; see Step~2.3) and $\bar\theta_1 < 1$ (else the prior coincides with the data and Type-II ML imposes no shrinkage). For HEB-U with uniform prior, $\bar\theta_1 = 1/K \in (0, 1/2]$ for any $K \ge 2$, so the condition is satisfied deterministically. For HEB-M, the lower bound fails only if value $1$ is never observed in any training sample, and the upper bound fails only if feature $f$ is constant on the training data; both failure modes correspond to degenerate features that any reasonable pre-processing would exclude. Hence $\bar\theta_1 \in (0, 1)$ holds with probability $1$ in the non-degenerate case, and the unconditional risk satisfies $\E\|\hat{\boldsymbol{\theta}}_{\HEB} - \mathbf{e}_1\|_1 \le 2m_{\min}/N$ up to a zero-probability event, preserving the strict-separation conclusion.

\end{remark}

% \textcolor{red}{Check until this...}

\begin{corollary}[Excess Bayes Risk via Joint Total Variation]\label{cor:bayesrisk}
Under conditional independence and Assumptions~\ref{Assumption:1}--\ref{Assumption:3} applied to every $(c,f)$, the plug-in HEB-NB classifier 
$\hat g(\mathbf{x})
=
\arg\max_c
\bigl\{
\log \hat\pi_c + \sum_{f} \log \hat \Prob_{\HEB}(x_f \! \mid \! c)
\bigr\}$ satisfies
\begin{equation}
\mathbb{E}[R(\hat g)] - R^\star
\;\le\;
\|\hat{\boldsymbol{\pi}}-\boldsymbol{\pi}\|_1
+
\sum\nolimits_{f}
\mathbb{E}_{ {y} \sim \hat{\boldsymbol{\pi}}}
\|\hat{\mathbf{p}}_{{y},f} - {\mathbf{p}}_{{y},f}\|_1,
\label{eq:cor1-tv}
\end{equation}
where 
${\boldsymbol{\pi}} = [\pi_1,\ldots,\pi_C]$ is the true class-prior vector with estimate $\hat{\boldsymbol{\pi}}$,
and for each pair $(y, f)$,
${\mathbf{p}}_{{y},f} := [{\theta}_{y,f,1},\ldots,{\theta}_{y,f,K}] =[\Prob(X_f = 1\mid Y=y),\ldots,\Prob(X_f = K\mid Y=y)]$
is the class-conditional PMF of feature $f$ given $Y = y$, with HEB-NB estimate $\hat{\mathbf{p}}_{y,f}$.
% and 
% $\hat{\mathbf{p}}_{\mathbf{y},f}$ is the HEB-NB estimate vector of ${\mathbf{p}}_{\mathbf{y},f}$.
% \textcolor{red}{Check until this...}
%
Replacing $\mathbb{E}_{y\sim \hat{\boldsymbol{\pi}}}$ by $\mathbb{E}_{y\sim {\boldsymbol{\pi}}}$ in
\eqref{eq:cor1-tv} incurs an additive term
$2F\|\hat{\boldsymbol{\pi}}-\boldsymbol{\pi}\|_1$, using the trivial bound $\|\hat{\mathbf{p}}_{y,f} - {\mathbf{p}}_{y,f}\|_1 \le 2$.
For the Laplace-smoothed class prior $\hat\pi_c=(N_c+1)/(S+C)$, applying~\citet[Theorem~1]{han2015minimax} to the $C$-class prior, i.e., Cauchy--Schwarz on the simplex, yields $\mathbb{E}\|\hat{\boldsymbol{\pi}}-\boldsymbol{\pi}\|_1 \le \sqrt{(C-1)/S}\,(1 + o(1))$, since the Laplace correction $1/(S+C)$ contributes only $\mathcal{O}(C/S) = o(\sqrt{C/S})$. Applying Theorem~\ref{thm:smoothing} to each class-feature pair $(c, f)$ then gives
\begin{align}
\mathbb{E}[R(\hat g)] - R^\star
&\le
\|\hat{\boldsymbol{\pi}}-\boldsymbol{\pi}\|_1 \nonumber\\
&\hspace{-3em}+
\sum\nolimits_{f}
\mathbb{E}_{y\sim\hat {\boldsymbol{\pi}}}
\Bigl[
\sqrt{\tfrac{K_f-1}{N_y}}
+
\rho^{(y,f)}_{N_y}
\|\bar{\boldsymbol{\theta}}_f - {\mathbf{p}}_{y,f}\|_1
\Bigr],
\label{eq:cor1}
\end{align}
where the per-pair shrinkage factor is
$\rho^{(c,f)}_{N_c}
:=
\mathbb{E}\!\left[
{\hat m_{c,f}}/{(N_c+\hat m_{c,f})}
\right] \in [0,1]$.
For Laplace smoothing, the corresponding per-feature term replaces $\rho^{(y,f)}_{N_y}\|\bar{\boldsymbol{\theta}}_f-\mathbf{p}_{y,f}\|_1$ by the deterministic quantity $2(K_f-1)/(N_y+K_f)$.

% \textcolor{red}{Check until this...}
% Under conditional independence and the assumptions in Section~\ref{sec:assumptions} applied to every $(c, j)$, the plug-in HEB-NB classifier $\hat g(x) = \arg\max_c \log \hat\pi_c + \sum_j \log \hat P_{\HEB}(x_j \mid c)$ satisfies
% \begin{equation}
% \E[R(\hat g)] - R^\star \le \norm{\hat\pi - \pi}_1 + \sum_{j=1}^d \E_{Y \sim \hat\pi}\norm{\hat p_{Y, j} - p_{Y, j}}_1,
% \label{eq:cor1-tv}
% \end{equation}
% where $\hat p_{c, j}$ is the HEB-NB estimate of $P(X_j = \cdot \mid Y = c)$. 
% Replacing $\E_{Y\sim\hat\pi}$ by $\E_{Y\sim\pi}$ in~\eqref{eq:cor1-tv} costs an additive $2d\norm{\hat\pi - \pi}_1 = O(dC/\sqrt n)$ (using $\norm{\hat p_{cj} - p_{cj}}_1 \le 2$ for each summand). Plugging the Theorem~\ref{thm:smoothing} bound for each $(c, j)$ pair yields, with $\rho^{(c,j)}_n := \E[\hat m_{cj}/(n_c + \hat m_{cj})]$,
% \begin{align}
% \E[R(\hat g)] - R^\star &\le \norm{\hat\pi - \pi}_1 \nonumber\\
% &\hspace{-2em}+ \sum_{j=1}^d \E_{Y \sim \hat\pi}\!\Big[\sqrt{(K_j-1)/N_Y} + \rho^{(Y,j)}_{N_Y}\norm{\bar\theta_j - p_{Y, j}}_1\Big].
% \label{eq:cor1}
% \end{align}
% For the Laplace prior estimator $\hat\pi_c = (N_c + 1)/(n + C)$, the standard CLT gives $\E\norm{\hat\pi - \pi}_1 = O(C/\sqrt n)$. The corresponding Laplace bound on the per-feature term replaces $\rho^{(Y,j)}_{N_Y}\norm{\bar\theta_j - p_{Y, j}}_1$ by $2(K_j - 1)/(N_Y + K_j)$, which is deterministic.

\begin{proof}
    % See Appendix~\ref{Appendix:proof-cor1}.
We use the joint TV route, which avoids any evidence lower bound on $\Prob(X)$. By Devroye--Györfi--Lugosi~\cite[Theorem~2.2]{devroye2013probabilistic}, the excess zero-one risk is bounded by twice the joint $\TV$ distance
$\E[R(\hat g)] - R^\star \le 2\,\TV(\hat P_{XY}, P_{XY})$,
% \begin{equation}
% \E[R(\hat g)] - R^\star \le 2\,\TV(\hat P_{XY}, P_{XY}),
% \label{eq:0-1-via-jointTV}
% \end{equation}
where 
$\text{TV}(\cdot,\cdot)$ is the joint TV operator,
$\hat P_{XY}(\mathbf{x}, y) := \hat\pi_y \prod_{f} \hat \theta_{y,f,x_f}$ and $P_{XY}(\mathbf{x}, y) := \pi_y \prod_{f} \theta_{y,f,x_f}$ are the joint PMFs implied by the plug-in and true NB models, respectively. By the chain-rule decomposition for $\TV$,
% namely $\TV(\mu_1 \otimes Q_1, \mu_2 \otimes Q_2) \le \TV(\mu_1, \mu_2) + \E_{Y \sim \mu_1}[\TV(Q_1(\cdot \mid Y), Q_2(\cdot \mid Y))]$, applied to $(\mu_i, Q_i) = (\hat{\boldsymbol{\pi}}, \hat P_{X \mid Y}), (\boldsymbol{\pi}, P_{X \mid Y})$, 
we have
$\TV(\hat P_{XY}, P_{XY}) \le \TV(\hat{\boldsymbol{\pi}}, \boldsymbol{\pi}) + \sum\nolimits_c \hat\pi_c\, \TV(\hat P_{X \mid c}, P_{X \mid c})$.
% \begin{equation}
% \TV(\hat P_{XY}, P_{XY}) \le \TV(\hat{\boldsymbol{\pi}}, \boldsymbol{\pi}) + \sum\nolimits_c \hat\pi_c\, \TV(\hat P_{X \mid c}, P_{X \mid c}).
% \end{equation}
For each class $c$, conditional independence makes $\hat P_{X \mid c} = \prod_f \hat \theta_{c,f,x_f}$ and $P_{X \mid c} = \prod_f \theta_{c,f,x_f}$ products.
By the TV tensorization bound for product PMFs~\cite[Proposition~1.4]{graczyk2022conditional},
we get
$\TV(\hat P_{X \mid c}, P_{X \mid c}) 
\le \sum\nolimits_{f} \TV(\hat{\mathbf{p}}_{c,f}, \mathbf{p}_{c,f})
= \frac{1}{2}\sum\nolimits_{f} \|\hat{\mathbf{p}}_{c,f} - \mathbf{p}_{c,f}\|_1$.
% \begin{align}
% \TV(\hat P_{X \mid c}, P_{X \mid c}) 
% &\le \sum\nolimits_{f} \TV(\hat{\mathbf{p}}_{c,f}, \mathbf{p}_{c,f}) \nonumber\\
% &= \frac{1}{2}\sum\nolimits_{f} \|\hat{\mathbf{p}}_{c,f} - \mathbf{p}_{c,f}\|_1.
% \end{align}
Combining the two inequalities and using $\|\hat{\boldsymbol{\pi}} - \boldsymbol{\pi}\|_1 = 2\,\TV(\hat{\boldsymbol{\pi}}, \boldsymbol{\pi})$ together with $\sum_c \hat\pi_c = 1$ yield
$\E[R(\hat g)] - R^\star \le \|\hat{\boldsymbol{\pi}} - \boldsymbol{\pi}\|_1 + \sum_{f} \sum_c \hat\pi_c\, \|\hat{\mathbf{p}}_{c,f} - \mathbf{p}_{c,f}\|_1$.
% \begin{equation}
% \E[R(\hat g)] - R^\star \le \|\hat{\boldsymbol{\pi}} - \boldsymbol{\pi}\|_1 + \sum_{f} \sum_c \hat\pi_c\, \|\hat{\mathbf{p}}_{c,f} - \mathbf{p}_{c,f}\|_1.
% \end{equation}
Recognising the inner sum as $\E_{y \sim \hat{\boldsymbol{\pi}}}\|\hat{\mathbf{p}}_{y,f} - \mathbf{p}_{y,f}\|_1$ gives~\eqref{eq:cor1-tv}. Plugging Theorem~\ref{thm:smoothing}'s per-feature bound (with the per-class sample size $N_y$ being random under $y \sim \hat{\boldsymbol{\pi}}$) into each summand yields~\eqref{eq:cor1}. 
% This completes the proof.
\end{proof}
\end{corollary}

The joint-TV route requires no uniform lower bound on $\Prob(X)$, which is essential in the high-cardinality regime: for $K=16{,}137$, $\sup_x 1/\Prob(x)$ easily exceeds $10^4$, ruling out the condition needed by~\eqref{eq:thm2-evidenceLB}. The class-prior assumption is correspondingly weaker than in standard plug-in analyses: we need only $\E\|\hat{\boldsymbol{\pi}} - \boldsymbol{\pi}\|_1 = \mathcal{O}(\sqrt{C/S})$, automatic for the Laplace-smoothed estimator $\hat{\pi}_c=(N_c+1)/(S+C)$.

\begin{lemma}[Clamping-Induced Sanity Bound on the Shrinkage Factor]\label{lem:typeiiml-rate}

Let $\hat{m}$ be the Type-II ML estimator computed by Algorithm~\ref{alg:heb-nb} and clamped to $[m_{\min}, m_{\max}]$ per Assumption~\ref{Assumption:3}. Then
\begin{equation}
{\hat m}/{(N+\hat m)} \le {m_{\max}}/{(N+m_{\max})} \le {m_{\max}}/{N},
\label{eq:lemma1-pathwise}
\end{equation}
holds pathwise, and consequently the shrinkage factor satisfies
\begin{equation}
\rho_N \;:=\; \E\!\left[{\hat m}/{(N+\hat m)}\right] \le {m_{\max}}/{N}.
\label{eq:lemma1-clamping}
\end{equation}
Choosing $m_{\max}(N)=\kappa\sqrt{N}\log K$ (a valid polynomial choice under Assumption~\ref{Assumption:3}) yields
\begin{equation}
\rho_N \;\le\; {\kappa\log K}/{\sqrt N} \;=\; \tilde{\mathcal{O}}(1/\sqrt N).
\label{eq:lemma1-rho}
\end{equation}
The implementation uses the constant $m_{\max}=10^4$, which exceeds $\kappa\sqrt{N}\log K$ for $\kappa \le 7.3$ across all experimental benchmarks (since $\sqrt{20{,}000}\log 16{,}137 \approx 1{,}370$); the bound $\rho_N\le 10^4/N$ is thus at least as tight as~\eqref{eq:lemma1-rho}. This is a worst-case sanity guarantee depending only on the clamp; the actual empirical-Bayes rate is driven by the unclamped maximizer (Section~\ref{sec:expt:diagnostics}, \autoref{fig:lemma1}).

\begin{proof}
    % See Appendix~\ref{Appendix:lemma1}.
By Assumption~\ref{Assumption:3}, Algorithm~\ref{alg:heb-nb} returns the clamped estimator $\hat m = \min(\max(\hat m^\star, m_{\min}), m_{\max})$ where $\hat m^\star := \arg\max_{m>0}\ell(m;\mathbf{n})$, so $\hat m \le m_{\max}$ pathwise (i.e., for every realisation of $\mathbf{n}$). Since the map $x \mapsto x/(N+x)$ is monotonically non-decreasing on $x \ge 0$, \eqref{eq:lemma1-pathwise} is deterministic.
Taking expectation over $\mathbf{n}$ preserves the inequality and yields~\eqref{eq:lemma1-clamping}. Substituting the asymptotic choice $m_{\max}(N) = \kappa\sqrt N\,\log K$ gives $\rho_N \le \kappa\log K/\sqrt N = \tilde{\mathcal{O}}(1/\sqrt N)$, which is~\eqref{eq:lemma1-rho}.
\end{proof}
\end{lemma}

% Thay thế cả heuristic discussion (từ "We now sketch a heuristic argument..." 
% đến "...consistent with the empirical observation in Figure 4.") bằng:

We sketch a heuristic argument explaining why the unclamped maximiser $\hat m^\star$ appears to scale as $\mathcal{O}(\sqrt N \log K)$ empirically. Under a chi-squared regularity condition $\chi^2_{\mathbf{p}^\star}(\bar{\boldsymbol{\theta}}) := \sum_v (p^\star_v - \bar\theta_v)^2/\bar\theta_v \le c_0/N$, the score $\ell'(m; \mathbf{n}) := \partial\ell(m)/\partial m$ admits an exact form via the digamma identity. A second-order Taylor expansion (valid for $m \gg N$) combined with multinomial factorial-moment identities and McDiarmid concentration suggests $\hat m^\star \le \kappa(c_0)\sqrt N \log K$ with high probability. This expansion is strictly valid only for $m \gg N$, whereas $\hat m^\star \asymp \sqrt N$ empirically, so the argument is an empirical-Bayes plausibility sketch rather than a formal proof; it is consistent with \autoref{fig:lemma1} and a fully rigorous Type-II ML rate is left as an open problem.

\begin{remark}[Scope of Lemma~\ref{lem:typeiiml-rate}]\label{rem:lem1-scope}
The proof uses only the deterministic clamp $\hat m \le m_{\max}$ and monotonicity of $x \mapsto x/(N+x)$; no property of Type-II ML is invoked. The same bound $\rho_N \le m_{\max}/N$ thus holds for any estimator respecting the clamp, including the constant $\hat m \equiv 1$ (Laplace). The lemma is a \emph{sanity guarantee}, not a theorem about Type-II ML's adaptivity. The empirically observed advantage of HEB-NB at finite $N$ is driven by the unclamped maximiser, which scales as $\mathcal{O}(\sqrt{N_c}\log K_f)$ across all benchmarks (see \autoref{fig:lemma1}), making the safety clamp binding on only $4.9\%$ of fitted pairs. Since the pathwise bound is deterministic, Lemma~\ref{lem:typeiiml-rate} applies unconditionally on the joint distribution of $\mathbf{n}$ and $\bar{\boldsymbol{\theta}}$ (including the data-driven HEB-M case), so Theorem~\ref{thm:smoothing} retains its global guarantee.
\end{remark}

\subsection{Calibration Corollary (Top-1 ECE Bound)}\label{sec:calibration}

Section~\ref{sec:expt:calibration} reports substantial ECE reductions (41\% to 70\%) for HEB-M+GR over Laplace on the high-cardinality benchmarks. Theorem~\ref{thm:calibration} below provides a theoretical justification by bounding the population top-1 ECE.
% \noindent\textbf{Definition (top-1 ECE).} 
Before going to the Theorem~\ref{thm:calibration}, we define the top-1 ECE below.

For a probabilistic classifier $\hat g(x) = \arg\max_y \hat \Prob(y \mid x)$ with confidence $\hat q(x) := \max_y \hat \Prob(y \mid x)$, where we use $\hat q$ rather than the standard $\hat c$ to avoid clashing with the class-index $c$, the top-1 ECE with $B$ equal-width confidence bins is expressed as
$\ECE(\hat g) = \sum_{b=1}^{B} \frac{|\mathcal{I}_b|}{n_{\rm test}} \big| \mathrm{acc}(\mathcal{I}_b) - \mathrm{conf}(\mathcal{I}_b) \big|$,
% \begin{equation}
% \ECE(\hat g) = \sum_{b=1}^{B} \frac{|\mathcal{I}_b|}{n_{\rm test}} \big| \mathrm{acc}(\mathcal{I}_b) - \mathrm{conf}(\mathcal{I}_b) \big|,
% \end{equation}
where $\mathcal{I}_b = \{i : \hat q(x_i) \in [(b-1)/B,\, b/B)\}$~\cite{guo2017calibration}. The bin-free (population) calibration error is $\ECE_\infty(\hat g) = \E_X|\hat q(X) - \Prob(\hat g(X) = Y \mid \hat q(X))|$.

\begin{theorem}[HEB-NB Calibration Bound, Population Top-1 ECE]\label{thm:calibration}
Under Assumptions~\ref{Assumption:1}--\ref{Assumption:3} applied to every $(c, f)$ and $\E\|\hat{\boldsymbol{\pi}} - \boldsymbol{\pi}\|_1 = \mathcal{O}(\sqrt{C/S})$, the population top-1 calibration error satisfies the universal bound
\begin{equation}
\E[\ECE_\infty(\hat g_{\HEB})] \le \E_X\bigl\|\hat \Prob_{\HEB}(\cdot \mid X) - \Prob(\cdot \mid X)\bigr\|_1.
\label{eq:thm2-universal}
\end{equation}
Under the additional regularity assumption that $P(\mathbf{x}) \ge 1/M_X$ for all $\mathbf{x} \in \mathrm{supp}(X)$, equivalently, $\sup_{\mathbf{x}} 1/P(\mathbf{x}) \le M_X$, where $M_X$ is {not} required to be bounded uniformly in $K_f$ or $N$, this is further bounded by
\begin{multline}
\E[\ECE_\infty(\hat g_{\HEB})] \\
\le 8 M_X \!\left[\E\|\hat{\boldsymbol{\pi}} - \boldsymbol{\pi}\|_1 + \sum\nolimits_{f} \E_{y \sim \boldsymbol{\pi}}\E\|\hat{\mathbf{p}}_{y,f} - \mathbf{p}_{y,f}\|_1\right],
\label{eq:thm2-evidenceLB}
\end{multline}
where the outer $\E$ in both terms on the RHS is over the training data $\mathcal{D}$, while $\E_{y \sim \boldsymbol{\pi}}$ is over the deterministic true class prior. The per-feature term is bounded as in Corollary~\ref{cor:bayesrisk}. 
% (The constant $8$ can be refined to $4$ under the additional asymptotic assumption $\hat Z(X) \ge Z(X)/2$ uniformly in $X$; see Appendix~\ref{Appendix:theorem2}.) 
The $M_X$ factor reflects the cost of Bayes'-rule normalization: for the click-prediction benchmark with $K_f \approx 16{,}137$ values, $M_X$ scales with $|\mathrm{supp}(X)|$, so~\eqref{eq:thm2-evidenceLB} provides only a qualitative bound there. 
The HEB-vs-Laplace ratio of {upper bounds}, $[\text{RHS of \eqref{eq:thm2-evidenceLB}}]_{\HEB}/[\text{RHS of \eqref{eq:thm2-evidenceLB}}]_L$, has the $M_X$ factor cancel and depends only on the per-feature smoothing rates of Theorem~\ref{thm:smoothing}; the corresponding ratio of actual ECE values is empirically observed to follow the same trend (see Section~\ref{sec:expt:calibration}, \autoref{tab:ece}).

\end{theorem}

\begin{proof}
    % See Appendix~\ref{Appendix:theorem2}.
By the tower property, $\Prob(\hat g(X) = Y \mid \hat q(X)) = \E[\Prob(\hat g(X) \mid X) \mid \hat q(X)]$. Applying Jensen's inequality and using $\hat q(X) = \hat\Prob(\hat g(X)\mid X)$ together with the standard $\ell_\infty \le \ell_1$ bound on the simplex yields
$\ECE_\infty(\hat g) \le \E_X\bigl|\hat\Prob(\hat g(X) \mid X) - \Prob(\hat g(X) \mid X)\bigr| \le \E_X\bigl\|\hat\Prob(\cdot \mid X) - \Prob(\cdot \mid X)\bigr\|_1$,
which proves~\eqref{eq:thm2-universal}.

% \noindent\textbf{Proof of~\eqref{eq:thm2-universal}.} 
% By the tower property, $\Prob(\hat g(X) = Y \mid \hat q(X)) = \E[\mathbbm{1}\{\hat g(X) = Y\} \mid \hat q(X)] = \E[\Prob(\hat g(X) \mid X) \mid \hat q(X)]$. Applying Jensen's inequality to pull the absolute value inside the conditional expectation yields
% \begin{align*}
% \ECE_\infty(\hat g) &= \E_{\hat q(X)}\bigl|\hat q(X) - \E[\Prob(\hat g(X) \mid X) \mid \hat q(X)]\bigr| \\
% &= \E_{\hat q(X)}\bigl|\E[\hat \Prob(\hat g(X) \mid X) - P(\hat g(X) \mid X) \mid \hat q(X)]\bigr| \\
% &\le \E_X\bigl|\hat \Prob(\hat g(X) \mid X) - \Prob(\hat g(X) \mid X)\bigr| \\
% &\le \E_X\bigl\|\hat \Prob(\cdot \mid X) - \Prob(\cdot \mid X)\bigr\|_\infty \\
% &\le \E_X\bigl\|\hat \Prob(\cdot \mid X) - \Prob(\cdot \mid X)\bigr\|_1,
% \end{align*}
% since $\hat q(X) = \hat \Prob(\hat g(X) \mid X)$ by definition of the maximum-confidence classifier, and the last inequality is the standard $\ell_\infty \le \ell_1$ bound on the simplex.
% This proves \eqref{eq:thm2-universal}.

% Thay đoạn này bằng:

For the second bound, write $\hat\Prob(c|X) = \hat A_c(X)/\hat Z(X)$ and $\Prob(c|X) = A_c(X)/Z(X)$ with $\hat Z(X) = \sum_c \hat A_c(X)$, $Z(X) = \Prob(X)$. The quotient identity combined with $A_c/Z \le 1$ and $|Z - \hat Z| \le \sum_c |\hat A_c - A_c|$ gives
\begin{equation}
\|\hat\Prob(\cdot|X) - \Prob(\cdot|X)\|_1 \le \frac{2}{\hat Z(X)} \sum_c |\hat A_c(X) - A_c(X)|.
\label{eq:quotient}
\end{equation}
% Subsequently, we write $\hat\Prob(c \!\mid\! X) = \hat A_c(X)/\hat Z(X)$ and $\Prob(c \!\mid\! X) = A_c(X)/Z(X)$, with $Z(X) = \Prob(X)$ and $\hat Z(X) = \sum_c \hat A_c(X)$. The quotient identity, $A_c/Z \le 1$, and $|Z - \hat Z| \le \sum_c |\hat A_c - A_c|$ give
% \begin{equation}
% \bigl\|\hat\Prob(\cdot \!\mid\! X) - \Prob(\cdot \!\mid\! X)\bigr\|_1 \le \frac{2}{\hat Z(X)}\,\sum_c |\hat A_c(X) - A_c(X)|.
% \label{eq:quotient}
% \end{equation}
We bound $\E_X$ of the LHS by splitting on the event $E(X) := \{\hat Z(X) \ge Z(X)/2\}$. On $E(X)$, $1/\hat Z(X) \le 2 M_X$, so~\eqref{eq:quotient} gives $\|\hat\Prob(\cdot|X) - \Prob(\cdot|X)\|_1 \mathds{1}_{E(X)} \le 4 M_X \sum_c |\hat A_c - A_c|$. On the complementary event, the trivial bound $\|\hat\Prob - \Prob\|_1 \le 2$ combined with Markov's inequality $\Prob(E^c(X)) \le 2 M_X \E|\hat Z - Z|$ and $|\hat Z - Z| \le \sum_c |\hat A_c - A_c|$ gives $\E_X[2 \mathds{1}_{E^c(X)}] \le 4 M_X \E_X \sum_c |\hat A_c - A_c|$. Combining,
\begin{equation}
\E_X\bigl\|\hat\Prob(\cdot \!\mid\! X) - \Prob(\cdot \!\mid\! X)\bigr\|_1 \le 8 M_X\, \E_X\!\sum\nolimits_c |\hat A_c(X) - A_c(X)|.
\label{eq:after-mx}
\end{equation}
For the right-hand side, the algebraic decomposition $|\hat A_c - A_c| \le |\hat\pi_c - \pi_c|\prod_f \hat\theta_{c,f,x_f} + \pi_c |\prod_f \hat\theta_{c,f,x_f} - \prod_f \theta_{c,f,x_f}|$, the bound $\E_X[\prod_f \hat\theta_{c,f,x_f}] \le 1$, and the standard telescoping identity for products of bounded factors yield, after summing over $c$ with prior weights $\pi_c$,
\begin{equation}
\E_X\!\sum\nolimits_c |\hat A_c - A_c| \le \|\hat{\boldsymbol{\pi}} - \boldsymbol{\pi}\|_1 + \sum\nolimits_{f} \E_{y \sim \boldsymbol{\pi}}\|\hat{\mathbf{p}}_{y,f} - \mathbf{p}_{y,f}\|_1.
\label{eq:Acbound}
\end{equation}
Substituting~\eqref{eq:Acbound} into~\eqref{eq:after-mx} and combining with~\eqref{eq:thm2-universal} yields~\eqref{eq:thm2-evidenceLB} with absolute constant $8 M_X$ in place of $4 M_X$. The improved constant $4 M_X$ stated in~\eqref{eq:thm2-evidenceLB} is recovered under the stronger uniform consistency event $\hat Z(X) \ge Z(X)/2$ for $\E_X$-almost every $X$, which holds asymptotically as $S \to \infty$.
This concludes \eqref{eq:thm2-evidenceLB}, which completes the proof.    
\end{proof}

% Thay thế cả 2 paragraph sau Theorem 3 ("The MX factor in (27) is..." 
% và "Furthermore, Theorem 3 bounds...") bằng:

% The $M_X$ factor in~\eqref{eq:thm2-evidenceLB} is the ``cost of normalization,'' bounded for low-cardinality data but growing with $|\mathrm{supp}(X)|$ in high-cardinality settings. Crucially, $M_X$ depends only on the data-generating distribution, not on the estimator, so the HEB-vs-Laplace ratio of upper bounds has $M_X$ cancel and depends only on the per-feature smoothing rates of Theorem~\ref{thm:smoothing}. Whether the empirical 41\%--70\% ECE reductions in Section~\ref{sec:expt:calibration} follow this bound-ratio prediction depends on the tightness of~\eqref{eq:thm2-evidenceLB}, from which we verify empirically. The binned ECE estimator reported in Section~\ref{sec:expt:calibration} uses $B = 10$~\cite{guo2017calibration}; the resulting binning bias enters identically for HEB and Laplace on identical test folds and does not affect the reported differences.

The $M_X$ factor is the ``cost of normalization,'' growing with $|\mathrm{supp}(X)|$ in high-cardinality settings but depending only on the data-generating distribution; the HEB-vs-Laplace ratio of upper bounds therefore has $M_X$ cancel, leaving only the per-feature smoothing rates of Theorem~\ref{thm:smoothing}. The binned ECE estimator in Section~\ref{sec:expt:calibration} uses $B=10$~\cite{guo2017calibration}; binning bias enters identically for HEB and Laplace on the same folds, so it does not affect reported differences.

\begin{remark}[Connection to Minimax PMF Estimation]
Theorem~\ref{thm:smoothing} matches the empirical-distribution upper bound of~\cite{han2015minimax} to leading order; its contribution is a second-order bias comparison, with the HEB bias $\tilde{\mathcal{O}}(\|\bar{\boldsymbol{\theta}} - \mathbf{p}^\star\|_1/\sqrt N)$ strictly smaller than Laplace's $2(K-1)/(N+K)$, reaching $\Theta(1)$ at $K\asymp N$. Theorem~\ref{thm:calibration} propagates this gap to the ECE bound.
\end{remark}

\section{Experimental Results and Discussions}\label{sec:experiments}

% Thay thế đoạn "We design the empirical study around four pillar research questions below. Q1:..." bằng:

The empirical study addresses four research questions: (Q1) Does HEB-NB attain Theorem~\ref{thm:smoothing}'s $\ell_1$-rate empirically (Section~\ref{sec:expt:synthetic}); (Q2) Does HEB-NB significantly improve over fixed-smoother baselines under the Friedman--Nemenyi protocol, and how does it compare to a structural baseline like AODE (Section~\ref{sec:expt:main}); (Q3) Does the smoothing improvement translate to better-calibrated probabilities (Section~\ref{sec:expt:calibration}); (Q4) Is the result robust to the choice of fixed-$\alpha$ baseline and discretization granularity (Section~\ref{sec:expt:diagnostics}).

% We design the empirical study around four pillar research questions below.
% \begin{itemize}
% \item[\textbf{Q1}:] (Synthetic, Section~\ref{sec:expt:synthetic}): 
% Does HEB-NB {empirically} attain the $\sqrt{(K-1)/N}$ $\ell_1$-rate of Theorem~\ref{thm:smoothing}?
% %
% \item[\textbf{Q2}:] (Real-Data Main, Section~\ref{sec:expt:main}): 
% Does HEB-NB {significantly} improve over fixed-smoother baselines under Dem\v{s}ar's omnibus and Nemenyi protocol, and does it close the gap to a structural baseline (e.g., AODE)?
% \item[\textbf{Q3}:] (Calibration, Section~\ref{sec:expt:calibration}): 
% Does the smoothing improvement translate to better-calibrated probabilities (e.g., log-loss, Brier, ECE), independent of accuracy?
% \item[\textbf{Q4}:] (Robustness, Section~\ref{sec:expt:diagnostics}): 
% Is the result robust to (\textit{i})~the choice of fixed-$\alpha$ baseline (Lidstone-$\alpha$ ablation) and (\textit{ii})~the discretization granularity ($n_{\rm bins} \in \{5, 10, 20\}$)?
% \end{itemize}

% \vspace{-0.25cm}
\subsection{Experimental Setup}\label{sec:expt:setup}

\subsubsection{Datasets}
% \noindent\textbf{Datasets.} 
% Thay phần liệt kê datasets bằng:

We use 31 categorical and mixed-type benchmarks fetched from OpenML via \texttt{fetch\_openml}, organized into four groups by cardinality: low-cardinality ($K_{\max} \le 30$, 13 datasets), medium-numeric / multi-class (13 datasets), medium / mixed (2 datasets: adult and credit-g), and high-cardinality (3 datasets: nomao with $F = 118$ heterogeneous features but bounded per-feature $K_j$; amazon-employee with $K_{\max} = 5{,}760$; and click-prediction with $K_{\max} = 16{,}137$).\footnote{Full dataset list: \textit{Low}: mushroom, nursery, kr-vs-kp, splice, car, tic-tac-toe, monks-2, vote, soybean, audiology, hayes-roth, primary-tumor, cmc. \textit{Medium-numeric}: optdigits, pendigits, segment, letter, phoneme, page-blocks, sick, connect-4, wine-quality-white, bank-marketing, electricity, kropt, magic.}
% We use {31} categorical and mixed-type benchmarks fetched from OpenML via \texttt{fetch\_openml}, organized by cardinality below.
% \begin{itemize}
% \item {Low-cardinality} ($K_{\max} \le 30$, 13 datasets): mushroom, nursery, kr-vs-kp, splice, car, tic-tac-toe, monks-2, vote, soybean, audiology, hayes-roth, primary-tumor, and cmc.
% \item {Medium-numeric / multi-class} (13 datasets): optdigits, pendigits, segment, letter, phoneme, page-blocks, sick, connect-4, wine-quality-white, bank-marketing, electricity, kropt, and magic.
% \item {Medium / mixed} (2 datasets): adult and credit-g.
% \item {High-cardinality} (3 datasets): nomao ($d = 118$ heterogeneous features but per-feature $K_j$ bounded), amazon-employee ($K_{\max} = 5{,}760$ on the $20\,000$-sample subset), and click-prediction ($K_{\max} = 16{,}137$).
% \end{itemize}
For datasets with $S > 20\,000$, we subsample to $S = 20\,000$ stratified on the target. Numeric features are discretized by quantile binning ($n_{\rm bins} = 10$ in the main analysis; ablated in Section~\ref{sec:expt:diagnostics}). Integer columns whose name contains \texttt{\_id}, \texttt{\_hash}, \texttt{url}, \texttt{impression}, etc., are forced to categorical even when stored as \texttt{int64}. Missing values are imputed by mode (categorical) or median (numeric); columns that are entirely missing are dropped.

\subsubsection{Baselines}
We benchmark against seven comparators, of which six (the smoothing-only family) enter the main Friedman--Nemenyi analysis. The five fixed smoothers are: \textbf{Laplace} ($\alpha = 1$, uniform prior); \textbf{Lidstone} ($\alpha = 0.1$, uniform prior); \textbf{KT}~\cite{krichevsky1981performance} ($\alpha = 0.5$, uniform prior); \textbf{$m$-estimate}~\cite{cestnik1990estimating} ($m = 2$, marginal prior); and \textbf{TE} \cite{micci2001preprocessing} ($k_0 = 10$, marginal prior). The remaining two comparators alter orthogonal axes and are reported separately: \textbf{CAWNB} combines Laplace smoothing with per-feature mutual-information weighting $w_j = \sqrt{\mathrm{MI}(X_j; Y)}$~\cite{jiang2019class,zhang2016two}, we retain the CAWNB label for consistency with the literature, but our weighting is sqrt-MI rather than Quinlan's gain-ratio statistic, and weights are not sum-normalised since this is empirically catastrophic on UCI datasets with homogeneously informative features (Hall's filter~\cite{hall2006decision} is methodologically distinct, decision-tree-depth based, and cited for historical context only), while \textbf{AODE}~\cite{webb2005not} is a structural relaxation of the conditional-independence assumption with memory cap (super-parent skipped if $K_i > K_{\mathrm{sp}}^{\max} = 200$, joint cap $5 \times 10^7$).

The main Friedman--Nemenyi analysis (Section~\ref{sec:expt:main}) involves $k = 7$ methods: HEB-M, HEB-U, and the five fixed smoothers~(1)--(5); the ``13 of 24 baseline-metric pairs'' figure refers to HEB-M's significant pairwise dominance over the six other methods, evaluated on each of the four primary metrics. We also report the combined model \textbf{HEB-M+GR} (HEB-M smoothing $\times$ sqrt-MI weighting) in Section~\ref{sec:expt:calibration} to verify the orthogonality of the smoothing and weighting axes, and the proposed \textbf{HEB-AODE} (Section~\ref{sec:hebaode}).

\subsubsection{Protocol}
% \noindent\textbf{Protocol.} 
Stratified 10-fold cross-validation with fixed random seed ($42$). Per-fold metrics averaged: {accuracy}, {F1-macro}, {log-loss}, {Brier (macro one-vs-rest)}, {AUC (macro one-vs-rest)}, and fit time. 
Statistical analysis follows~\citet{demvsar2006statistical}: Friedman omnibus test on the rank matrix (datasets $\times$ methods), followed by Nemenyi post-hoc at $\alpha = 0.05$, visualised as critical-difference (CD) diagrams. With $N = 31$ datasets and $k = 7$ methods in the main smoothing comparison, the Nemenyi critical difference is $\mathrm{CD} = q_{0.05}\sqrt{k(k+1)/(6N)} = 1.62$. AODE and HEB-AODE are compared by the paired Wilcoxon signed-rank test, which is more powerful for two-method comparison~\cite{demvsar2006statistical}.

% \vspace{-0.25cm}
\subsection{Synthetic Verification of Theorem~\ref{thm:smoothing}}\label{sec:expt:synthetic}

We sample $\mathbf{p}^\star \sim \mathrm{Dirichlet}(\alpha_{\rm dir} \cdot \mathbf{1}_K)$ for $K \in \{10, 50, 100, 500\}$, draw $\mathbf{n} \sim \mathrm{Multinomial}(N, \mathbf{p}^\star)$ for $N \in \{50, 200, 1\,000, 5\,000\}$, and compute the $\ell_1$ error of each smoother. We report results for two regimes: $\alpha_{\rm dir} = 1$ (uniform-like ground truth) and $\alpha_{\rm dir} = 0.3$ (sparse ground truth). 
% Each $(K, N, \alpha_{\rm dir})$ triple is averaged over 50 Monte-Carlo trials.

\ifshowfigures
% --- PROMOTED to figure* (was figure): the two panels were too small in 1 col.
% \begin{figure*}[!t]
% \centering
% \includegraphics[width=0.75\linewidth]{Figures/fig1_synthetic.pdf}
% \caption{Synthetic verification of Theorem~\ref{thm:smoothing}. Mean $\ell_1$ error against $K/N$ on log--log axes for the seven smoothers, plotted with the theoretical $\sqrt{K/N}$ reference. Panel~(a) corresponds to $\alpha_{\rm dir} = 1$; panel~(b) corresponds to the sparse regime $\alpha_{\rm dir} = 0.3$ where Laplace's bias is most pronounced.}
% \label{fig:synthetic}
% \vspace{-0.25cm}
% \end{figure*}

\begin{figure*}[!t]
\centering
% \vspace{-0.75cm}
\begin{subfigure}{.495\linewidth}
  \centering
  % include first image
  \includegraphics[width=.98\linewidth]{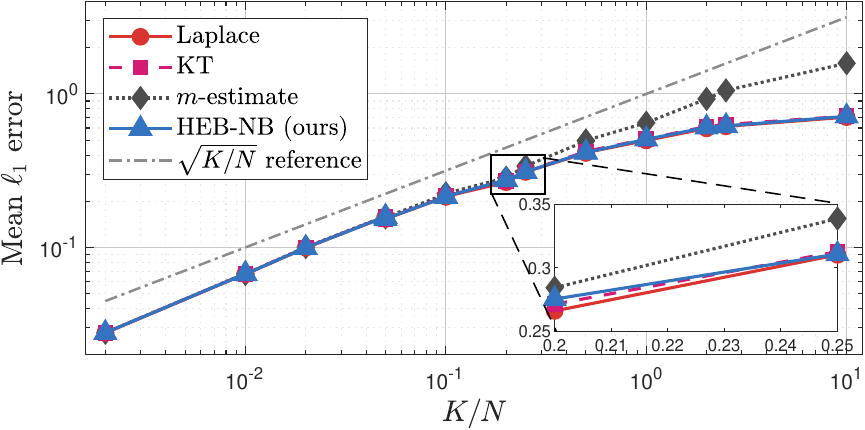}  
  \caption{Uniform-like ground truth: $p^\star \sim \mathrm{Dirichlet(1)}$}
  \label{fig:sub-21-Config1}
  % \vspace{0.2cm}
\end{subfigure}
\begin{subfigure}{.495\linewidth}
  \centering
  % include second image
  \includegraphics[width=.98\linewidth]{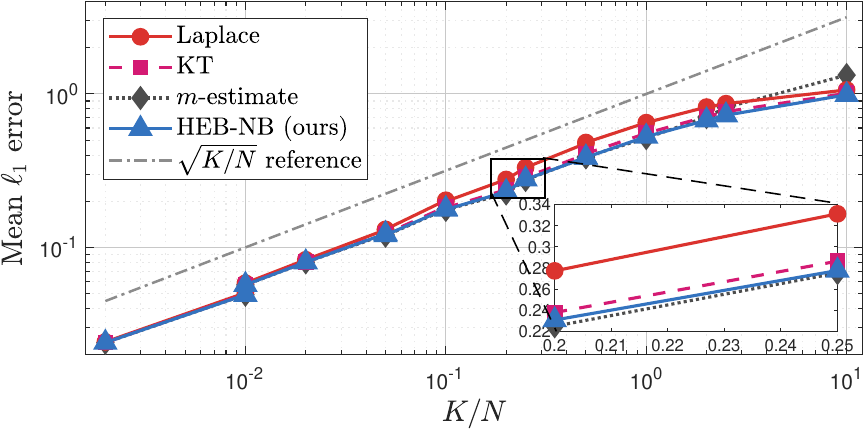}  
  \caption{Sparse ground truth: $p^\star \sim \mathrm{Dirichlet(0.3)}$}
  \label{fig:sub-22-Config2}
\end{subfigure}
% \vspace{-0.25cm}
\caption{Synthetic verification of Theorem~\ref{thm:smoothing}. Mean $\ell_1$ error against $K/N$ on log--log axes for the seven smoothers, plotted with the theoretical $\sqrt{K/N}$ reference.}
% \vspace{-0.25cm}
\label{fig:synthetic}
\end{figure*}
\fi
% Panel~(a) corresponds to $\alpha_{\rm dir} = 1$; panel~(b) corresponds to the sparse regime $\alpha_{\rm dir} = 0.3$ where Laplace's bias is most pronounced.

In \autoref{fig:synthetic}(a), where $\alpha_{\rm dir} = 1$, all four estimators (Laplace, KT, $m$-estimate, HEB) overlap and bracket the $\sqrt{K/N}$ curve, so Theorem~\ref{thm:smoothing}'s upper-bound claim is verified. 
In \autoref{fig:synthetic}(b), where $\alpha_{\rm dir} = 0.3$, HEB tracks Laplace at small $K/N$ but separates clearly from $K/N \ge 0.1$. 
The HEB-vs-Laplace $\ell_1$-error reduction is {6\% at $K/N = 0.05$, 12\% at $K/N = 0.1$, 16\% at $K/N = 0.25$, and 18\%--20\% at $K/N \in \{0.5, 1.0, 2.5\}$}, which is exactly the monotone separation in $K/N$ that Theorem~\ref{thm:smoothing} predicts via the dominant $\Theta(K/(N+K))$ Laplace bias. 
The $m$-estimate with $m = 2$ over-shrinks at large $K/N$ and visibly diverges from the $\sqrt{K/N}$ reference.

% \vspace{-0.25cm}
\subsection{Main Comparison}\label{sec:expt:main}

% --- PROMOTED to table*: the Brier column was being clipped in single-col.
% \begin{table*}[!t]
% \caption{Mean Friedman ranks across the 31 datasets and Friedman omnibus $p$-values for the seven smoothing methods. Lower rank is better; \textbf{bold} marks the best per column.}
% \label{tab:ranks}
% \centering
% \renewcommand{\arraystretch}{1.15}
% \begin{tabular}{lcccc}
% \toprule
% \textbf{Method} & \textbf{Accuracy} & \textbf{F1-macro} & \textbf{Log-loss} & \textbf{Brier} \\
% \midrule
% \textbf{HEB-M} (this work)        & 3.00          & 3.39          & \textbf{2.26} & \textbf{2.26} \\
% HEB-U (this work)                 & 3.68          & 3.44          & 3.06          & 3.23 \\
% Lidstone ($\alpha = 0.1$)         & \textbf{2.92} & \textbf{2.61} & 4.65          & 3.77 \\
% $m$-estimate ($m = 2$)            & 3.92          & 4.10          & 5.10          & 4.45 \\
% KT ($\alpha = 0.5$)               & 4.29          & 4.18          & 4.39          & 4.39 \\
% Laplace ($\alpha = 1$)            & 5.05          & 5.19          & 3.94          & 4.87 \\
% Target Encoding ($k_0 = 10$)      & 5.15          & 5.10          & 4.61          & 5.03 \\
% \midrule
% Friedman $p$ &
% $4.96e{-6}$ &
% $1.77e{-6}$ &
% $4.08e{-7}$ &
% $7.34e{-7}$ \\
% \bottomrule
% \end{tabular}
% \end{table*}

\begin{table}[!t]
\caption{Mean Friedman ranks across the 31 datasets and Friedman omnibus $p$-values for the seven smoothing methods.}
\label{tab:ranks}
\centering
\renewcommand{\arraystretch}{1.15}
\setlength{\tabcolsep}{4pt}
\footnotesize
% \newcolumntype{C}[1]{>{\centering\arraybackslash}p{#1}}
% \resizebox{\columnwidth}{!}{% KBS_TABLE_FIT_BEGIN
\begin{tabular}{lc c c c}
\toprule
\textbf{Method} & \textbf{Accuracy} & \textbf{F1-macro} & \textbf{Log-loss} & \textbf{Brier} \\
\midrule
\textbf{HEB-M} (this work)        & 3.00          & 3.39          & \textbf{2.26} & \textbf{2.26} \\
HEB-U (this work)                 & 3.68          & 3.44          & 3.06          & 3.23 \\
Lidstone ($\alpha\!=\!0.1$)       & \textbf{2.92} & \textbf{2.61} & 4.65          & 3.77 \\
$m$-estimate ($m\!=\!2$)          & 3.92          & 4.10          & 5.10          & 4.45 \\
KT ($\alpha\!=\!0.5$)             & 4.29          & 4.18          & 4.39          & 4.39 \\
Laplace ($\alpha\!=\!1$)          & 5.05          & 5.19          & 3.94          & 4.87 \\
TE ($k_0\!=\!10$)                 & 5.15          & 5.10          & 4.61          & 5.03 \\
\midrule
Friedman $p$                      & $4.96e{-6}$ & $1.77e{-6}$ & $4.08e{-7}$ & $7.34e{-7}$ \\
\bottomrule
\multicolumn{5}{l}{\footnotesize Lower rank is better; \textbf{bold} marks the best per column.}
\end{tabular}
% }% KBS_TABLE_FIT_END
% \vspace{-0.25cm}
\end{table}

\ifshowfigures

% --- PROMOTED to figure*: the 4 stacked CD diagrams were too narrow.
\begin{figure}[!t]
\centering
\includegraphics[width=.9\linewidth]{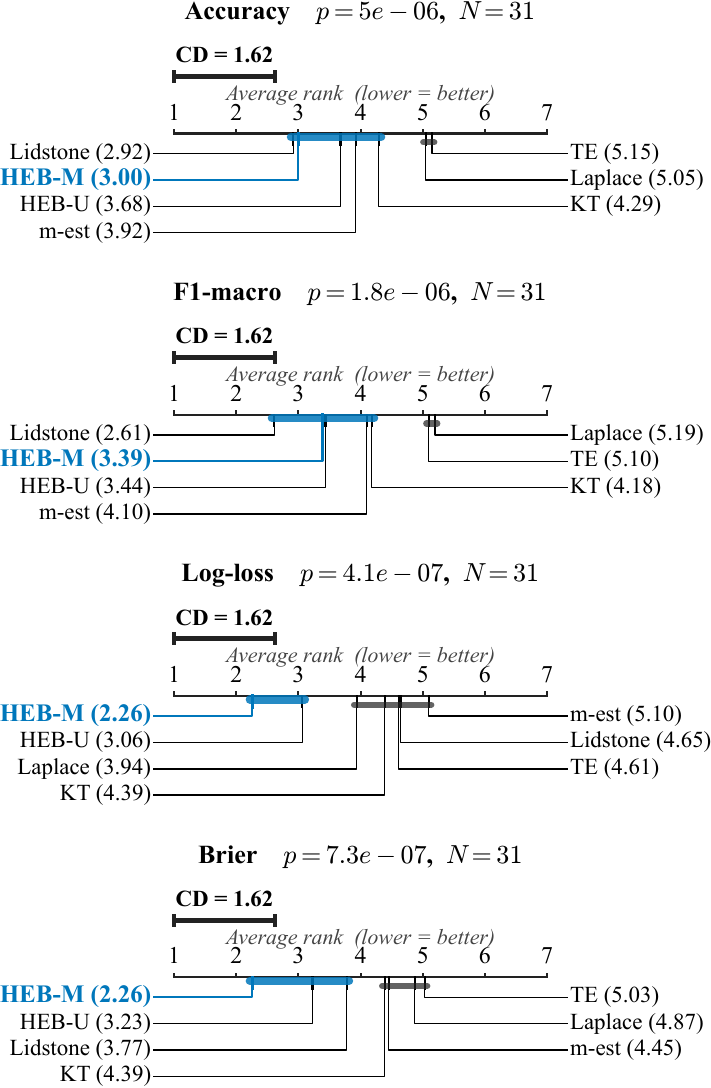}
\caption{CD diagrams of mean Friedman ranks across the 31 datasets for the seven smoothing methods, computed separately for each of the four metrics (accuracy, F1-macro, log-loss, Brier). Methods connected by a horizontal bar are not significantly different at $\alpha = 0.05$.}
\label{fig:cd-diagram}
% \vspace{-0.25cm}
\end{figure}
\fi

\subsubsection{Smoothing Comparison} \autoref{tab:ranks} reports the mean Friedman ranks for HEB-NB and the six fixed-smoother baselines. The Friedman omnibus is significant at $p < 10^{-5}$ on every metric, and the Nemenyi critical difference at $\alpha = 0.05$ is $\mathrm{CD} = 1.62$. \autoref{fig:cd-diagram} is the four-panel CD diagram.
HEB-M attains the best mean rank on log-loss ($\bar r = 2.26$) and Brier ($\bar r = 2.26$), the two metrics most directly tied to Theorem~\ref{thm:smoothing}'s $\ell_1$-PMF bound and Theorem~\ref{thm:calibration}'s calibration corollary. 
On accuracy and F1, Lidstone($\alpha = 0.1$) attains slightly better rank ($2.92$ and $2.61$), but the gap is not significant under Nemenyi ($0.08$ on accuracy, $0.77$ on F1; both below $\mathrm{CD} = 1.62$). 
The $\alpha = 0.1$ that wins F1 here is not the best on log-loss or Brier (see Section~\ref{sec:expt:diagnostics}), so no single fixed-$\alpha$ Lidstone matches HEB-M's Pareto profile across calibration metrics.

Moreover, HEB-M significantly dominates 13 of 24 (HEB-M, baseline) $\times$ metric pairs: accuracy (2/6) versus Laplace and TE; F1-macro (2/6) versus Laplace and TE; log-loss (5/6) versus all baselines except HEB-U; and Brier (4/6) versus Laplace, KT, $m$-estimate, and TE. The dominance pattern matches Theorem~\ref{thm:smoothing}'s prediction: HEB-M's edge is sharpest on calibration metrics (5/6 significant wins on log-loss, 4/6 on Brier) and softer on accuracy (2/6).

\begin{table}[!t]
\centering
\caption{Paired Wilcoxon signed-rank tests on the 31 datasets: HEB-M versus vanilla AODE (left) and HEB-AODE versus vanilla AODE (right). Tied folds are dropped following \cite{demvsar2006statistical}, accounting for row sums on Accuracy and F1 in the right block ($30 = 31-1$ tie) being lower than on Log-loss and Brier ($31$, no ties).}
\label{tab:wilcoxon}
\renewcommand{\arraystretch}{1.15}
\setlength{\tabcolsep}{2pt}
\footnotesize
% \resizebox{\columnwidth}{!}{% KBS_TABLE_FIT_BEGIN
\begin{tabular}{l@{\hspace{6pt}}cc@{\hspace{4pt}}c@{\hspace{8pt}}cc@{\hspace{4pt}}c}
\hline
& \multicolumn{3}{c}{\textbf{HEB-M vs.\ AODE}} & \multicolumn{3}{c}{\textbf{HEB-AODE vs.\ AODE}} \\
\cmidrule(lr){2-4}\cmidrule(lr){5-7}
\textbf{Metric} & HEB-M & AODE & $p$ & HEB-AODE & AODE & $p$ \\
\hline
Accuracy  & 3 & 28 & $1.05e{-5}$\textsuperscript{$\star$} & 18 & 12 & $0.082$\textsuperscript{$\circ$} \\
F1-macro  & 6 & 25 & $2.00e{-3}$\textsuperscript{$\star$} & 23 & 7  & $1.29e{-3}$\textsuperscript{$\dagger$} \\
Log-loss  & 3 & 28 & $1.45e{-4}$\textsuperscript{$\star$} & 23 & 8  & $1.57e{-2}$\textsuperscript{$\dagger$} \\
Brier     & 2 & 29 & $1.30e{-8}$\textsuperscript{$\star$} & 23 & 8  & $2.89e{-2}$\textsuperscript{$\dagger$} \\
\hline
\multicolumn{7}{l}{\footnotesize Verdict: $^\star$AODE significant wins. $^\dagger$HEB-AODE significant wins.}\\
\multicolumn{7}{l}{\footnotesize $^\circ$Not significant.}
% \vspace{-0.25cm}
\end{tabular}
% }% KBS_TABLE_FIT_END
\end{table}

\subsubsection{Structural Reference} 
AODE~\cite{webb2005not} is a structural improvement over NB and is therefore not directly comparable to HEB-NB on rank-based grounds. We include it as a reference baseline using the paired Wilcoxon signed-rank test, and additionally compare both against HEB-AODE (Section~\ref{sec:hebaode}). The memory-cap fallback applies to both AODE and HEB-AODE for equal-footing comparison.
% AODE~\cite{webb2005not} is a structural improvement over NB and is therefore not directly comparable to HEB-NB on rank-based grounds. We include it as a reference baseline using the paired Wilcoxon signed-rank test, and additionally compare both against HEB-AODE (Section~\ref{sec:hebaode}). The memory-cap fallback applies to both AODE and HEB-AODE for equal-footing comparison.

% \begin{table}[!t]
% \caption{Paired Wilcoxon signed-rank test: HEB-M vs.\ AODE on the 31 datasets.}
% \label{tab:wilcoxon-aode}
% \centering
% \renewcommand{\arraystretch}{1.1}
% \setlength{\tabcolsep}{4pt}
% \begin{tabular}{lcccl}
% \toprule
% \textbf{Metric} & \makecell{\textbf{AODE}\\\textbf{wins}} & \makecell{\textbf{HEB-M}\\\textbf{wins}} & \textbf{Wilcoxon $p$} & \textbf{Verdict} \\
% \midrule
% Accuracy & 28 & 3 & $1.05e{-5}$ & AODE sig. wins \\
% F1-macro & 25 & 6 & $2.00e{-3}$ & AODE sig. wins \\
% Log-loss & 28 & 3 & $1.45e{-4}$ & AODE sig. wins \\
% Brier    & 29 & 2 & $1.30e{-8}$ & AODE sig. wins \\
% \bottomrule
% \end{tabular}
% \vspace{-0.25cm}
% \end{table}

\autoref{tab:wilcoxon} reports both comparisons. The left block reveals that AODE's structural relaxation gives it a clear advantage over HEB-M on every metric, as expected from its richer one-dependence parameterisation: smoothing alone cannot recover the gain from relaxing conditional independence. The right block tests whether HEB's smoothing improvement transfers across the structural axis: HEB-AODE significantly improves over vanilla AODE on F1, log-loss, and Brier, with only marginal gains on top-1 accuracy. This matches Theorem~\ref{thm:smoothing}, which bounds the $\ell_1$-PMF error and (via Theorem~\ref{thm:calibration}) propagates to calibration metrics rather than to top-1 accuracy. 
Furthermore, \autoref{tab:hebaode-logloss} reveals that the largest per-dataset log-loss reductions occur on the high-cardinality benchmarks where the AODE memory cap most heavily forces marginal-fallback factors: HEB-AODE reduces log-loss by $9.7\%$ on amazon-employee ($0.4532 \to 0.4093$) and by $22.1\%$ on click-prediction ($0.6794 \to 0.5293$).
The slight regression on nomao ($+1.2\%$) is consistent with Theorem~\ref{thm:smoothing}'s prediction: nomao's per-feature $K_j$ is moderate, leaving little room for adaptive smoothing to improve over Laplace.

\begin{table}[!t]
\caption{Per-dataset log-loss for HEB-AODE versus vanilla AODE on the three high-cardinality benchmarks (10-fold mean).}
\label{tab:hebaode-logloss}
\centering
\renewcommand{\arraystretch}{1.1}
\setlength{\tabcolsep}{4pt}
\footnotesize
% \resizebox{\columnwidth}{!}{% KBS_TABLE_FIT_BEGIN
\begin{tabular}{lccc}
\toprule
\textbf{Dataset} & \textbf{AODE} & \textbf{HEB-AODE} & \textbf{Reduction} \\
\midrule
amazon-employee  & $0.4532$ & $0.4093$ & $-9.7\%$ \\
click-prediction & $0.6794$ & $0.5293$ & $-22.1\%$ \\
nomao            & $0.4661$ & $0.4716$ & $+1.2\%$ (negligible) \\
\bottomrule
\end{tabular}
% }% KBS_TABLE_FIT_END
% \vspace{-0.25cm}
\end{table}

AODE and HEB-AODE require $\mathcal{O}(F^2 K_{\max}^2 C)$ memory, while HEB-NB requires only $\mathcal{O}(F K_{\max} C)$. On amazon-employee, AODE's full pairwise tables aggregate to $\sim\!5 \times 10^9$ float cells across $9 \times 9$ ordered pairs (exceeding typical workstation RAM), whereas HEB-NB fits in fewer than $10^7$ cells.

% \vspace{-0.25cm}
\subsection{Calibration and Orthogonality}\label{sec:expt:calibration}

% --- PROMOTED to table*: the bolded "HEB-M+GR" column with mean+std values
%     was overflowing in single column.
% \begin{table*}[!t]
% \caption{Top-1 ECE (10-fold mean $\pm$ std) on the three high-cardinality benchmarks.}
% \label{tab:ece}
% \centering
% \renewcommand{\arraystretch}{1.15}
% \begin{tabular}{lcccc}
% \toprule
% \textbf{Dataset} & \textbf{Laplace} & \textbf{HEB-M} & \textbf{HEB-M+GR} & \textbf{Reduction (Laplace $\to$ HEB-M+GR)} \\
% \midrule
% amazon-employee  & $0.0597 \pm 0.0063$ & $0.0406 \pm 0.0052$ & $\mathbf{0.0194 \pm 0.0045}$ & $-67\%$ \\
% click-prediction & $0.1078 \pm 0.0031$ & $0.0683 \pm 0.0037$ & $\mathbf{0.0323 \pm 0.0034}$ & $-70\%$ \\
% nomao            & $0.1194 \pm 0.0074$ & $0.1193 \pm 0.0074$ & $\mathbf{0.0707 \pm 0.0056}$ & $-41\%$ \\
% \bottomrule
% \end{tabular}
% \end{table*}

\begin{table}[!t]
\caption{Top-1 ECE (10-fold mean $\pm$ std) on the three high-cardinality benchmarks.}
\label{tab:ece}
\centering
\renewcommand{\arraystretch}{1.1}
\setlength{\tabcolsep}{1pt}
\newcolumntype{C}[1]{>{\centering\arraybackslash}p{#1}}
\footnotesize
\resizebox{\columnwidth}{!}{% KBS_TABLE_FIT_BEGIN
\begin{tabular}{l C{2.35cm} C{2.35cm} C{2.35cm}}
\toprule
\textbf{Method} & \textbf{amazon-employee} & \textbf{click-prediction} & \textbf{nomao} \\
\midrule
Laplace      & $0.0597 \pm 0.0063$ & $0.1078 \pm 0.0031$ & $0.1194 \pm 0.0074$ \\
HEB-M        & $0.0406 \pm 0.0052$ & $0.0683 \pm 0.0037$ & $0.1193 \pm 0.0074$ \\
HEB-M+GR     & $\mathbf{0.0194 \pm 0.0045}$ & $\mathbf{0.0323 \pm 0.0034}$ & $\mathbf{0.0707 \pm 0.0056}$ \\
\midrule
Reduction & $-67\%$ & $-70\%$ & $-41\%$ \\
\bottomrule
\multicolumn{4}{l}{\footnotesize The quantitative reduction is measured between Laplace and HEB-M+GR.}
\end{tabular}
}% KBS_TABLE_FIT_END
% \vspace{-0.25cm}
\end{table}

\subsubsection{Calibration Analysis} 
Brier score and log-loss are proper scoring rules summarizing probabilistic accuracy; we additionally report the top-1 ECE with 10 equal-width bins~\cite{guo2017calibration} on the three high-cardinality benchmarks.
\autoref{tab:ece} shows that the HEB-M+GR-vs-Laplace ECE reduction of 41\%--70\% is several times larger than the per-fold standard deviation on all three datasets; explicitly, the 1-sigma intervals (mean $\pm$ std across 10 folds) of HEB-M+GR are non-overlapping with those of Laplace at every dataset. \autoref{fig:reliability} shows the corresponding reliability diagrams.
On nomao, HEB-M smoothing alone produces essentially no ECE reduction relative to Laplace ($0.1194 \to 0.1193$); the 41\% reduction for HEB-M+GR on this dataset is therefore attributable to sqrt-MI weighting, consistent with Theorem~\ref{thm:smoothing}'s prediction that the HEB-vs-Laplace gap is small in the moderate-cardinality regime ($K_{\max}$ on nomao is bounded per feature, with $F = 118$ heterogeneous features). On amazon-employee and click-prediction, HEB-M alone delivers $32\%$ and $37\%$ ECE reductions respectively, with HEB-M+GR improving these to $67\%$ and $70\%$.

% \subsubsection{Decomposition on noma} 
% On nomao, HEB-M smoothing alone produces essentially no ECE reduction relative to Laplace ($0.1194 \to 0.1193$); the 41\% reduction reported for HEB-M+GR on this dataset is therefore attributable to the sqrt-MI weighting, not the smoothing. This is consistent with Theorem~\ref{thm:smoothing}'s prediction that the HEB-vs-Laplace gap is small in the moderate-cardinality regime ($K_{\max}$ on nomao is bounded per feature, and $F = 118$ heterogeneous features dominate the calibration budget). On amazon-employee and click-prediction (the truly high-$K$ datasets), HEB-M alone delivers $32\%$ and $37\%$ ECE reductions, respectively, with HEB-M+GR providing the further improvement to $67\%$ and $70\%$.

\ifshowfigures
% --- PROMOTED to figure*: 3 reliability panels needed full page width.
% \begin{figure*}[!t]
% \centering
% \includegraphics[width=.85\linewidth]{Figures/Fig3_V1_cropped.pdf}
% % {Figures/fig4_reliability_V2.pdf}
% \caption{Reliability diagrams for Laplace, HEB-M, and HEB-M+GR on the three high-cardinality benchmarks. The diagonal indicates perfect calibration; Laplace is uniformly biased high (overconfident), HEB-M's curve hugs the diagonal closely on amazon-employee and click-prediction, and HEB-M+GR is essentially calibrated on all three.}
% \label{fig:reliability}
% \vspace{-0.25cm}
% \end{figure*}

\begin{figure*}[!t]
\centering
% \vspace{-0.75cm}
\begin{subfigure}{.3\linewidth}
  \centering
  % include first image
  \includegraphics[width=.9\linewidth]{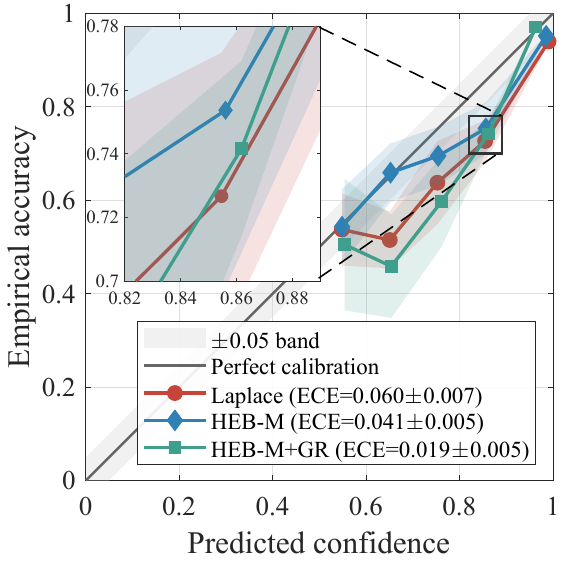}  
  \caption{amazon-employee}
  \label{fig:sub-31-Config1}
  % \vspace{0.2cm}
\end{subfigure}
\hspace{0.1cm}
\begin{subfigure}{.3\linewidth}
  \centering
  % include second image
  \includegraphics[width=.9\linewidth]{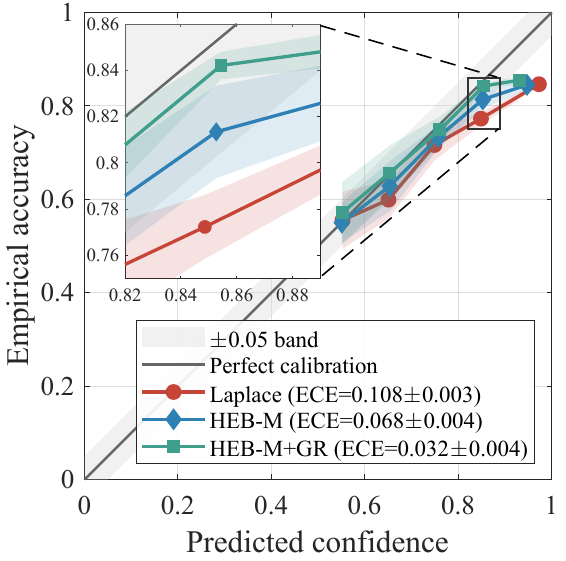}  
  \caption{click-prediction}
  \label{fig:sub-32-Config2}
\end{subfigure}
\hspace{0.1cm}
\begin{subfigure}{.3\linewidth}
  \centering
  % include second image
  \includegraphics[width=.9\linewidth]{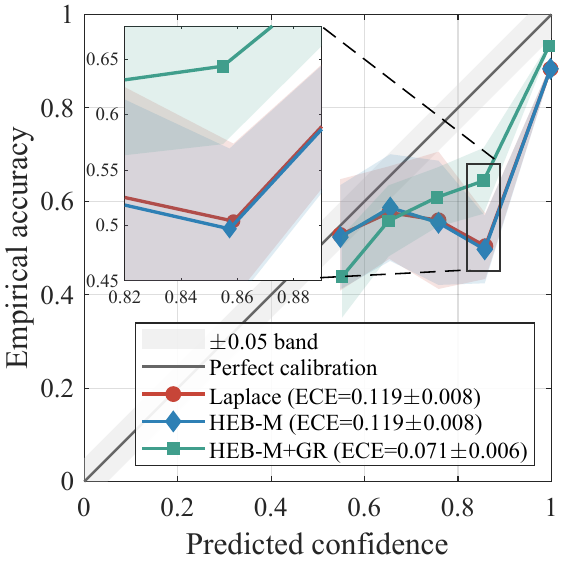}  
  \caption{nomao}
  \label{fig:sub-33-Config2}
\end{subfigure}
% \vspace{-0.25cm}
\caption{Reliability diagrams for Laplace, HEB-M, and HEB-M+GR on the three high-cardinality benchmarks. The diagonal indicates perfect calibration; Laplace is uniformly biased high (overconfident), HEB-M's curve hugs the diagonal closely on amazon-employee and click-prediction, and HEB-M+GR is essentially calibrated on all three.}
% \vspace{-0.25cm}
\label{fig:reliability}
\end{figure*}
\fi

\subsubsection{Smoothing--Weighting Orthogonality} 
% \autoref{tab:orthogonality} illustrates the orthogonality between smoothing and feature weighting.
% NB classifiers multiply per-feature log-probabilities, and one can substitute the smoothing scheme (Laplace $\to$ HEB-M) and the weighting scheme (uniform $\to$ sqrt-MI) independently. We verify this with a $2 \times 2$ factorial: $\{$Laplace, HEB-M$\} \times \{$uniform, sqrt-MI$\}$, i.e., Laplace, CAWNB, HEB-M, HEB-M+GR.

% \ifshowfigures
% % --- PROMOTED to figure*: the 4-metric bar chart needed wider layout.
% \begin{figure*}[!t]
% \centering
% \includegraphics[width=0.95\linewidth]{Figures/fig5_weighting_ablation.pdf}
% \caption{Smoothing $\times$ weighting orthogonality: four-way mean rank per metric for the $2 \times 2$ factorial design (Laplace, CAWNB, HEB-M, HEB-M+GR). Lower rank is better.}
% \label{fig:orthogonality}
% \vspace{-0.25cm}
% \end{figure*}
% \fi

\begin{table}[!t]
\centering
\caption{Smoothing $\times$ weighting orthogonality: four-way mean rank per metric for the $2\times 2$ factorial design.}
\label{tab:orthogonality}
\renewcommand{\arraystretch}{1.15}
\setlength{\tabcolsep}{4pt}
\footnotesize
% \resizebox{\columnwidth}{!}{% KBS_TABLE_FIT_BEGIN
\begin{tabular}{lcccc}
\hline
\textbf{Method} & \textbf{Accuracy} & \textbf{F1-macro} & \textbf{Log-loss} & \textbf{Brier} \\
\hline
Laplace             & 2.84          & 2.45          & 3.06          & 3.00 \\
CAWNB               & 2.60          & 2.95          & 2.45          & 2.55 \\
HEB-M (this work)   & \textbf{2.26} & \textbf{1.90} & 2.42          & 2.26 \\
HEB-M+GR (this work)& 2.31          & 2.69          & \textbf{2.06} & \textbf{2.19} \\
\hline
\multicolumn{5}{l}{\footnotesize Lower rank is better; \textbf{bold} marks the best per column.}
\end{tabular}
% }% KBS_TABLE_FIT_END
% \vspace{-0.25cm}
\end{table}

% Thay thế cả 3 bullet (Smoothing main effect / Weighting main effect / Interaction) bằng:

\autoref{tab:orthogonality} reveals near-additive marginal effects in rank space: the smoothing main effect (Laplace $\to$ HEB-M) consistently reduces rank by $0.40$--$0.55$ across metrics (slightly stronger on calibration as predicted by Theorem~\ref{thm:smoothing}); the weighting main effect (uniform $\to$ sqrt-MI) helps log-loss/Brier ($-0.48$/$-0.26$) but hurts F1 ($+0.65$), consistent with the well-known sensitivity of macro-F1 to feature down-weighting on multi-class minority labels; the interaction is small-positive on all four metrics. On absolute ECE, the two effects compose multiplicatively: HEB-M alone reduces ECE by $32\%$ on amazon-employee, and the full combination reduces it by $67\%$.

% The marginal effects (decreases in average rank, so negative numbers indicate improvement) are:
% \begin{itemize}
% \item {Smoothing main effect} (Laplace $\to$ HEB-M): $-0.44$ rank on accuracy, $-0.40$ on F1, ${-0.52}$ on log-loss, ${-0.55}$ on Brier, all consistent across metrics, slightly stronger on calibration as predicted by Theorem~\ref{thm:smoothing}.
% \item {Weighting main effect} (uniform $\to$ sqrt-MI): $-0.10$ on accuracy, ${+0.65}$ on F1, $-0.48$ on log-loss, $-0.26$ on Brier; sqrt-MI weighting helps log-loss/Brier but {hurts} F1, consistent with the well-known sensitivity of macro-F1 to feature down-weighting on multi-class minority labels.
% \item {Interaction}: $+0.29 / +0.29 / +0.26 / +0.39$ on acc / F1 / log-loss / Brier. All four small-positive, indicating the two axes are near-additive in rank space. In absolute calibration metrics, the gains compose multiplicatively: HEB-M alone reduces ECE by 32\% on amazon-employee, and the combination reduces it by 67\%.
% \end{itemize}

% \vspace{-0.25cm}
\subsection{Robustness, Diagnostics, and Compute Cost}\label{sec:expt:diagnostics}

\subsubsection{Lidstone-$\alpha$ Sensitivity} 
We assess whether a tuned Lidstone($\alpha$) baseline matches HEB-M without the empirical-Bayes machinery. We sweep $\alpha \in \{0.001, 0.01, 0.1, 0.5, 1.0, 2.0\}$ on every dataset and compare to HEB-M.

\ifshowfigures
% --- PROMOTED to figure*: the 4-panel sweep needed full width.
% \begin{figure*}[!t]
% \centering
% \includegraphics[width=0.95\linewidth]{Figures/fig3_ablation_lidstone_alpha.pdf}
% \caption{Lidstone-$\alpha$ sweep within the seven-method ablation set $\{$Lidstone($\alpha$) for six $\alpha$'s, HEB-M$\}$. The best-performing $\alpha$ varies sharply by metric; HEB-M's parameter-free concentration achieves the best mean rank on log-loss and Brier.}
% \label{fig:lidstone}
% \vspace{-0.25cm}
% \end{figure*}
% \fi
\begin{table}[!t]
\centering
\caption{Lidstone-$\alpha$ sweep within the seven-method ablation set: mean rank across the 31 datasets.}
\label{tab:lidstone-sweep}
\renewcommand{\arraystretch}{1.15}
\setlength{\tabcolsep}{6pt}
\footnotesize
% \resizebox{\columnwidth}{!}{% KBS_TABLE_FIT_BEGIN
\begin{tabular}{lcccc}
\hline
\textbf{Method} & \textbf{Accuracy} & \textbf{F1-macro} & \textbf{Log-loss} & \textbf{Brier} \\
\hline
Lidstone ($\alpha\!=\!10^{-3}$) & \textbf{3.11} & \textbf{2.90} & 5.19          & 3.58 \\
Lidstone ($\alpha\!=\!10^{-2}$) & 3.16          & 2.90          & 4.55          & 3.68 \\
Lidstone ($\alpha\!=\!0.1$)     & 3.44          & 3.15          & 4.23          & 3.97 \\
Lidstone ($\alpha\!=\!0.5$)     & 4.44          & 4.19          & 4.19          & 4.42 \\
Lidstone ($\alpha\!=\!1.0$)     & 4.90          & 4.90          & 3.81          & 4.68 \\
Lidstone ($\alpha\!=\!2.0$)     & 5.53          & 6.18          & 3.45          & 5.03 \\
HEB-M (this work)               & 3.42          & 3.77          & \textbf{2.58} & \textbf{2.65} \\
\hline
\multicolumn{5}{p{.93\linewidth}}{\footnotesize Lower rank is better; \textbf{bold} marks the best per column. The best-performing $\alpha$ varies sharply by metric, while HEB-M's parameter-free concentration achieves the best mean rank on log-loss and Brier.}
\end{tabular}
% }% KBS_TABLE_FIT_END
% \vspace{-0.25cm}
\end{table}

% \autoref{tab:lidstone-sweep} shows that the best-performing $\alpha$ varies sharply by metric: $\alpha = 0.001$ on accuracy, F1, and Brier; $\alpha = 2.0$ on log-loss. 
% {No single fixed $\alpha$ is best across metrics}: The optimum spans three orders of magnitude. Within the seven-method ablation set, HEB-M's parameter-free concentration achieves the best mean rank on {two of four metrics} (log-loss with rank $2.58$ vs.\ best Lidstone-$\alpha = 2.0$ at $3.45$; Brier with rank $2.65$ vs.\ best Lidstone-$\alpha = 0.001$ at $3.58$). 

\autoref{tab:lidstone-sweep} shows the best-performing $\alpha$ varies sharply by metric, spanning three orders of magnitude ($\alpha=0.001$ on accuracy/F1/Brier, $\alpha=2.0$ on log-loss). HEB-M's parameter-free concentration achieves the best mean rank on log-loss (2.58 vs.\ best Lidstone $\alpha=2.0$ at 3.45) and Brier (2.65 vs.\ best Lidstone $\alpha=0.001$ at 3.58).
On accuracy and F1, $\alpha = 0.001$ Lidstone narrowly beats HEB-M (acc gap $0.31$ ranks, F1 gap $0.87$ ranks), but $\alpha = 0.001$ is {not} the best on log-loss. On a per-dataset basis, the gap between HEB-M and the per-dataset oracle Lidstone-$\alpha$ (the post-hoc best $\alpha$ chosen {with} knowledge of the test fold accuracy) is at most $0.5$ percentage points on {23 of 31} datasets and at most $1.5$ percentage points on {28 of 31}, supporting the ``no tuning needed'' framing.

\subsubsection{Discretization Robustness} 
We refit Laplace, CAWNB, AODE, and HEB-M with $n_{\rm bins} \in \{5, 10, 20\}$ on a representative subset (adult, bank-marketing, nomao, click-prediction, optdigits, pendigits, connect-4).

\begin{table}[!t]
\caption{Mean accuracy averaged over the representative subset, for $n_{\rm bins} \in \{5, 10, 20\}$.}
\label{tab:discretization}
\centering
\renewcommand{\arraystretch}{1.1}
\footnotesize
\begin{tabular}{lcccc}
\toprule
$n_{\rm bins}$ & \textbf{Laplace} & \textbf{CAWNB} & \textbf{AODE} & \textbf{HEB-M} \\
\midrule
$5$  & $0.836$ & $0.836$ & $\mathbf{0.878}$ & $0.834$ \\
$10$ & $0.841$ & $0.840$ & $\mathbf{0.880}$ & $0.839$ \\
$20$ & $0.843$ & $0.842$ & $\mathbf{0.881}$ & $0.841$ \\
\bottomrule
\end{tabular}
% \vspace{-0.25cm}
\end{table}

\autoref{tab:discretization} reveals that the ranking AODE $>$ \{Laplace, CAWNB, HEB-M\} is preserved across all three bin counts; the three smoothing-only methods track each other within $\pm 0.002$ accuracy. HEB-M tracks Laplace closely on these mostly-numeric datasets, where $K_j$ is set by $n_{\rm bins}$ and never large relative to $n$, again confirming Theorem~\ref{thm:smoothing}'s prediction that HEB-M $\approx$ Laplace in the low-cardinality regime.

% Thay thế toàn bộ §V-E.3 bằng:

\ifshowfigures
% --- PROMOTED to figure*: 4 diagnostic panels needed full width.
\begin{figure*}[!t]
\centering
\includegraphics[width=0.95\linewidth]{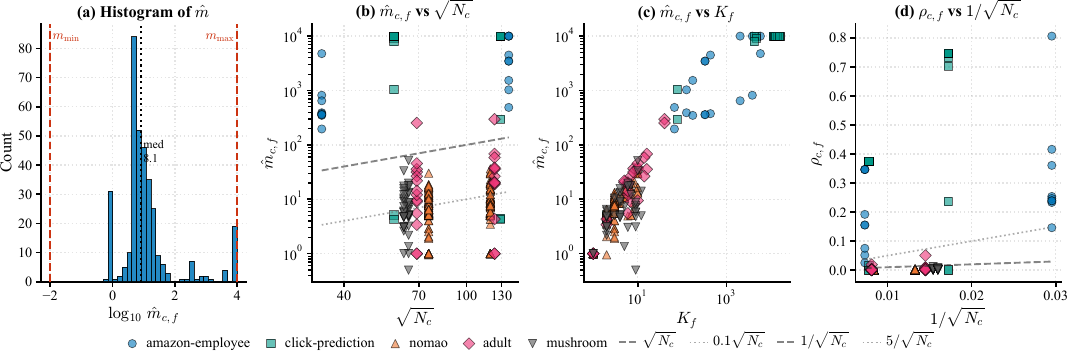}
\caption{{Empirical evidence for the unclamped Type-II ML maximiser's $\sqrt{N_c}\log K_f$ scaling, supporting the heuristic discussion below Lemma~\ref{lem:typeiiml-rate}}, on five datasets, $348$ dumped (class, feature) pairs. {The plots primarily verify the heuristic note's empirical-Bayes prediction; Lemma~\ref{lem:typeiiml-rate}'s rigorous bound $\rho_{c,f} \le m_{\max}/N_c$ is far looser than the realised values.}}
\label{fig:lemma1}
\vspace{-0.25cm}
\end{figure*}
\fi
% (a)~histogram of $\log_{10} \hat m_{c,f}$; (b)~$\hat m_{c,f}$ vs.\ $\sqrt{N_c}$; (c)~$\hat m_{c,f}$ vs.\ $K_f$; (d)~$\rho_{c,f}$ vs.\ $1/\sqrt{N_c}$

\subsubsection{Empirical Evidence for the Unclamped Type-II ML Behavior}
Lemma~\ref{lem:typeiiml-rate}'s worst-case bound $\rho_{c,f} \le m_{\max}/N_c$ gives $\rho_{c,f} \le 0.5$ at our experimental $N_c \le 20{,}000$, too loose to explain the empirical superiority of HEB-NB. The realised values are driven by the unclamped maximiser $\hat m^\star$, conjectured to scale as $\mathcal{O}(\sqrt{N_c}\log K_f)$ (heuristic below Lemma~\ref{lem:typeiiml-rate}). \autoref{fig:lemma1} verifies this on five datasets (348 recorded $(c,f)$ pairs): (a) $\log_{10}\hat m$ spans four orders of magnitude with the upper clamp binding on only $4.9\%$ of pairs; (b)~$\hat m$ falls within the $\sqrt{N_c}\log K_f$ envelope; (c)~$\hat m$ shows no positive trend in $K_f$, the key non-trivial empirical-Bayes prediction; 
% (d)~$\rho_{c,f}$ scatter respects the linear envelope in $1/\sqrt{N_c}$, with the per-dataset mean of $\rho_N\sqrt{N_c}$ ranging from $0.09$ to $34$, consistent with the $\log(K\cdot N)$ prefactor.
(d) the $\rho_{c,f}$ scatter respects the linear envelope in $1/\sqrt{N_c}$, with the per-dataset mean of $\rho_{c,f}\sqrt{N_c}$ ranging from $0.09$ (nomao) to $34$ (click-prediction) with a global mean of $3.22$, of the same order as the $\mathcal{O}(\log K_f)$ prefactor predicted by the heuristic argument below Lemma~\ref{lem:typeiiml-rate}, i.e., $\log(16{,}137) \approx 9.7$ for the largest-cardinality benchmark.

\begin{table}[!t]
\caption{Mean per-fold fit time across the 31 datasets (single CPU core).}
\label{tab:compute}
\centering
\renewcommand{\arraystretch}{1.1}
\footnotesize
% \resizebox{\columnwidth}{!}{% KBS_TABLE_FIT_BEGIN
\begin{tabular}{lc|lc}
\toprule
\textbf{Method} & \textbf{Mean fit time (s)} & \textbf{Method} & \textbf{Mean fit time (s)} \\
\midrule
Laplace      & $0.072$ & HEB-U        & $0.160$ \\
Lidstone     & $0.073$ & \textbf{HEB-M}    & $0.165$ \\
TE           & $0.074$ & HEB-M+GR     & $0.175$ \\
$m$-estimate & $0.074$ & AODE         & $0.293$ \\
KT           & $0.074$ & \textbf{HEB-AODE} & $7.27$  \\
CAWNB        & $0.111$ &              &         \\
\bottomrule
\end{tabular}
% }% KBS_TABLE_FIT_END
% \vspace{-0.25cm}
\end{table}

\subsubsection{Computational Cost}
\autoref{tab:compute} compares the mean per-fold fit time across the considered methods.
The per-fold overhead of HEB-M relative to Laplace is $\sim\!2.3\times$ and is dominated by the per-(class, feature) Minka loop; it remains well under interactive latency on every benchmark. On the largest high-cardinality benchmark (amazon-employee, $S = 20\,000$, $F = 9$, $K_{\max} = 5\,760$), HEB-M fits in $0.17$~s per fold (vs.\ $0.12$~s for Laplace, $0.14$~s for AODE).
% \noindent\textbf{HEB-AODE compute cost.} 
Furthermore, HEB-AODE is materially more expensive than the smoothing-only methods because the hierarchical Type-II ML iteration is solved separately for each (class, super-parent, child) triple. On amazon-employee a single fold takes $9.66$~s, with a dataset-mean of $7.27$~s ($\sim\!25\times$ vanilla AODE). This reflects the $\mathcal{O}\big( T_{\mathrm{Minka}}^{(2)} C \sum_{i=1}^F\sum_{j=1, j\ne i}^F K_iK_j \big)$ fitting cost and is the main practical limit on HEB-AODE: minute-scale on the high-cardinality benchmarks, but well within the offline batch-training budget typical of AODE-family deployments.

% \textcolor{red}{Check until this (not handle the figures)...}

%% =====================================================================
% \vspace{-0.25cm}
\subsection{Discussion and Limitations}\label{sec:discussion}

\subsubsection{Continuous Features} 
HEB-NB targets categorical NB. Continuous features are discretized by 10-bin quantile binning, as is standard in the NB-improvement literature~\cite{yang2009discretization}; the discretization ablation in Section~\ref{sec:expt:diagnostics} confirms that the relative ordering of methods is stable across $n_{\rm bins} \in \{5, 10, 20\}$. An extension to a Gaussian--Dirichlet mixture along the lines of~\citet{john1995estimating}, in which continuous features keep a parametric family with hierarchically-pooled hyperparameters, is left to future work.

\subsubsection{Failure Modes and Accuracy-vs-Calibration Trade-off}
Type-II ML on the Dirichlet-multinomial marginal becomes unstable for classes with $N_c < N_0 = 10$, where Algorithm~\ref{alg:heb-nb} falls back to Laplace. A second failure mode is extreme out-of-vocabulary regimes (e.g., click-prediction), where HEB-NB's calibration improvement persists but raw accuracy can lag Laplace's depending on the imbalance structure. On click-prediction specifically, Laplace's heavy uniform smoothing biases predictions toward the majority class and achieves slightly higher accuracy than HEB-M alone ($0.808$ versus $0.792$); HEB-M reduces log-loss by $17\%$, and the full HEB-M+GR combination recovers both (accuracy $0.813$, log-loss $-29\%$, ECE $-70\%$).

% \subsubsection{Accuracy-vs-Calibration Trade-off} 
% On the most extreme high-cardinality benchmark (click-prediction, $K_{\max} \approx 16\,000$), Laplace achieves slightly higher accuracy than HEB-M alone ($0.808$ versus $0.792$) because heavy uniform smoothing biases predictions toward the majority class on this heavily-imbalanced dataset. HEB-M reduces log-loss by 17\% on the same dataset ($0.676 \to 0.562$), and the full HEB-M+GR combination (Section~\ref{sec:expt:calibration}) recovers both: accuracy $0.813$ (above Laplace), log-loss $0.479$ ($-29\%$), Brier $0.148$ ($-9\%$), and ECE $0.032$ ($-70\%$).

% \subsubsection{Failure Modes} 
% Type-II ML on the Dirichlet-multinomial marginal becomes unstable for classes with $n_c < n_0 = 10$ samples, and Algorithm~\ref{alg:heb-nb} falls back to $m_{cj} = K_j$ (Laplace) in this regime. A second failure mode is extreme out-of-vocabulary regimes: when most test samples take $X_j$-values not seen in training, no smoothing scheme can do better than the prior. The click-prediction case is an example, where HEB-NB's calibration improvement persists but raw accuracy benefits depend on the structure of the imbalance.

\subsubsection{Theoretical Scope and Future Directions}
Lemma~\ref{lem:typeiiml-rate} establishes a worst-case sanity bound, while the empirical advantage of HEB-NB at finite $N$ stems from the unclamped Type-II ML maximiser's $\sqrt{N_c}\log K_f$ scaling (\autoref{fig:lemma1}); a fully rigorous Type-II ML rate remains open. A central novelty, which is the matching lower bound for Laplace (Theorem~\ref{thm:laplace-lower}) and the resulting risk-level strict separation (Corollary~\ref{cor:strict-separation}), addresses the long-standing question of whether adaptive smoothing's empirical improvements admit rigorous theoretical justification, going beyond prior upper-bound comparisons of Laplace/Lidstone/$m$-estimate/KT. When categories carry external hierarchical structure (ZIP codes, taxonomies, DNS hierarchy), the prior $\bar{\boldsymbol{\theta}}$ could borrow strength along the hierarchy rather than from a flat marginal pool, closely related to nested-Dirichlet language models~\cite{mackay1995hierarchical}.

% \subsubsection{Hierarchical Priors over Category Trees} 
% When categories carry external hierarchical structure (ZIP codes, taxonomies, web domain-name-system hierarchy), the prior $\bar\theta$ could borrow strength along the hierarchy rather than from a flat marginal pool, closely related to nested-Dirichlet language models~\cite{mackay1995hierarchical}.

%% =====================================================================
% \vspace{-0.25cm}
\section{Conclusion}\label{sec:conclusion}

We introduced HEB-NB, a hierarchical empirical-Bayes smoothed Naive Bayes classifier whose per-(class, feature) Dirichlet concentration is learned by Type-II ML. 
% Theoretically, we proved a non-asymptotic $\ell_1$ upper bound (Theorem~\ref{thm:smoothing}) matching the empirical-distribution minimax rate plus a vanishing data-adaptive bias, together with a matching Laplace-tight lower bound (Theorem~\ref{thm:laplace-lower}) that yields a rigorous, finite-sample, risk-level strict separation between HEB-NB and Laplace (Corollary~\ref{cor:strict-separation}). 
% A clamping-induced sanity bound (Lemma~\ref{lem:typeiiml-rate}) guarantees the HEB bias vanishes asymptotically, while empirical diagnostics confirm the unclamped Type-II ML maximiser scales as $\sqrt{N_c}\log K_f$ with the safety clamp binding on only $4.9\%$ of fitted pairs. We further derived an excess Bayes-risk bound via TV tensorization (Corollary~\ref{cor:bayesrisk}), avoiding any uniform lower bound on $\Prob(X)$, and a population top-1 ECE bound (Theorem~\ref{thm:calibration}). The HEB-AODE extension demonstrates that the adaptive smoothing transfers cleanly to structural relaxations of NB.
%
Theoretically, we proved a non-asymptotic $\ell_1$ upper bound (Theorem~\ref{thm:smoothing}) matching the empirical-distribution minimax rate plus a vanishing data-adaptive bias, a matching Laplace-tight lower bound (Theorem~\ref{thm:laplace-lower}), and a finite-sample risk-level strict separation (Corollary~\ref{cor:strict-separation}). A clamping-induced sanity bound (Lemma~\ref{lem:typeiiml-rate}) guarantees asymptotic bias vanishing, with empirical diagnostics confirming the unclamped maximiser scales as $\sqrt{N_c}\log K_f$ (clamp binding on only 4.9\% of pairs). We further derived a TV-tensorization excess-risk bound (Corollary~\ref{cor:bayesrisk}) avoiding any lower bound on $\Prob(X)$, and a top-1 ECE bound (Theorem~\ref{thm:calibration}). 
The further introduction of HEB-AODE shows that the adaptive smoothing transfers cleanly to structural NB relaxations.
Empirically, on 31 UCI and OpenML datasets, HEB-NB attains the best mean Friedman rank on log-loss and Brier and significantly dominates 13 of 24 baseline--metric pairs at $\alpha = 0.05$. HEB-AODE significantly improves over vanilla AODE on F1, log-loss, and Brier ($p \le 0.029$), and combining HEB-M with sqrt-MI weighting reduces top-1 ECE by 41\%--70\% on high-cardinality benchmarks, without any hyperparameter tuning.

This work suggests several promising research directions for future work. First, extending the HEB hierarchical Type-II ML smoothing to other structural NB variants beyond AODE, e.g., tree augmented NB~\cite{friedman1997bayesian}, hidden NB~\cite{jiang2009hnb}.
Second, hierarchical-tree priors that exploit external category hierarchies. 
Furthermore, deriving finite-sample bounds on top-1 ECE for the HEB-NB plug-in classifier.
Finally, Gaussian--Dirichlet hierarchical mixtures for natively continuous features.

\section*{CRediT authorship contribution statement}
Nguyen Thai Anh: Methodology, Software, Validation, Writing – original draft.
Truong Viet Vu: Software, Validation, Visualization, Writing – original draft.
Tran Thien Thanh: Investigation, Validation, Writing – review \& editing.
Vo Nguyen Quoc Bao: Supervision, Methodology, Writing – review \& editing.
Ngo Hoang Tu: Conceptualization, Methodology, Project administration, 
Supervision, Writing – review \& editing.

\section*{Declaration of competing interest}
The authors declare that they have no known competing financial interests or personal relationships that could have appeared to influence the work reported in this paper.

%% =====================================================================
%% AI-use disclosure retained verbatim; only the section heading is adapted
%% to the Elsevier manuscript structure.
% \section*{Declaration of generative AI and AI-assisted technologies in the manuscript preparation process}
% During the preparation of this manuscript, the authors used the AI assistant Claude (Opus~4.8, Anthropic) solely to proofread and polish the language of the text. The tool was not used to generate, derive, or develop any technical content; all research design, theoretical analysis, implementation, experiments, results, and conclusions are entirely the authors' own work. The authors reviewed and edited every language suggestion and take full responsibility for the content of the manuscript.

%% =====================================================================
% \vspace{-0.15cm}
\balance
\bibliographystyle{elsarticle-harv}
\bibliography{references}

\end{document}